\documentclass{tlp}

\usepackage{times}
\usepackage{soul}
\usepackage{url}
\usepackage[hidelinks]{hyperref}
\usepackage[utf8]{inputenc}

\usepackage[small]{caption}
\usepackage{graphicx}
\usepackage{amsmath}
\usepackage{booktabs}
\usepackage{algorithm}
\usepackage{algorithmic}
\usepackage{adjustbox}
\usepackage{multicol}
\usepackage[normalem]{ulem} 
\usepackage{xspace}
\usepackage{xcolor}
\usepackage{natbib}

\usepackage{tikz}
\usetikzlibrary{arrows,positioning,shapes.misc,shapes.geometric}

\newcommand{\HR}{\mathit{HD}}
\newcommand{\IC}{\mathit{IC}}
\newcommand{\HRc}{\mathcal{HD}}
\newcommand{\HDc}{\HRc}
\newcommand{\ICc}{\mathcal{IC}}

\DeclareMathOperator{\lfp}{lfp}

\newcommand {\tup}[1]      {{\langle #1 \rangle}}
\newcommand{\Xb}{\mathbf{X}}

\DeclareMathAlphabet\mathbfcal{OMS}{cmsy}{b}{n}
\newcommand{\modelsht}{\models_{\sf HT}}
\DeclareMathOperator*{\argmin}{\arg\!\min}

\newcommand{\conj}{\mathit{Conj}}
\usepackage{todo}

\tikzset{
    >=stealth',
    args/.style={circle,draw=black}
}

\newtheorem{definition}{Definition}
\newtheorem{proposition}{Proposition}
\newtheorem{corollary}{Corollary}

\newtheorem{lemma}[proposition]{Lemma}
\newtheorem{example}{Example}

\begin{document}

\lefttitle{Jesse Heyninck and Bart Bogaerts}

\jnlPage{1}{14}
\jnlDoiYr{2023}
\doival{10.1017/xxxxx}

\title[Non-deterministic AFT]{Non-deterministic approximation operators: ultimate operators, semi-equilibrium semantics and aggregates (full version)\thanks{This work was partially supported by  Fonds Wetenschappelijk Onderzoek -- Vlaanderen (project G0B2221N) and the Flemish Government (Onderzoeksprogramma Artifici\"ele Intelligentie (AI) Vlaanderen).}}

\begin{authgrp}
\author{\sn{Heyninck} \gn{Jesse}}
\affiliation{Open Universiteit, the Netherlands}
\author{\sn{Bogaerts} \gn{Bart}}
\affiliation{Vrij Universiteit Brussels, Belgium}
\end{authgrp}

\history{}

\maketitle

\begin{abstract}
Approximation fixpoint theory (AFT) is an abstract and general algebraic framework for studying the semantics of non-monotonic logics. In recent work, AFT was generalized to non-deterministic operators, i.e.\ operators whose range are sets of elements rather than single elements. In this paper, we make three further contributions to non-deterministic AFT: (1) we define and study ultimate approximations of non-deterministic operators, (2) we give an algebraic formulation of the semi-equilibrium semantics by Amendola, et al., and (3) we generalize the characterisations of disjunctive logic programs to disjunctive logic programs with aggregates.  

This is an extended version of our paper that will be presented at ICLP 2023 and will appear
in the special issue of TPLP with the ICLP proceedings.
\end{abstract}


\section{Introduction}
Knowledge representation and reasoning (KRR), by its very nature, is concerned with the study of a wide variety of languages and formalisms. 
In view of this, \emph{unifying} frameworks that allow for the language-independent study of aspects of KRR is essential.
One framework with strong unifying potential is \emph{approximation fixpoint theory} (AFT) \citep{denecker2000approximations}, a purely algebraic theory which was shown to unify the semantics of, among others, logic programming %
default logic and autoepistemic logic. %
The central objects of study of AFT are \emph{(approximating) operators} and their \emph{fixpoints}.
For logic programming for instance, it was shown that Fitting's three-valued immediate consequence operator is an approximating operator of Van Emden and Kowalski's two-valued immediate consequence operator and that all major semantics of (normal) logic programming can be derived directly from this approximating operator. 
Moreover, this observation does not only hold for logic programming: also for a wide variety of other domains, it is straightforward how to derive an approximating operator, and the major semantics can be recovered from that approximator using purely algebraic means (an overview is given by \cite{nondetAFTarxiv}). 
This has in turn inspired others to define the semantics of non-monotonic  formalisms \emph{directly} using AFT \citep{bogaerts2019weighted}, putting AFT forward not only as a framework to study existing semantics, but also as a framework to define them.
The advantage is that AFT-based semantics are guaranteed to follow well-established principles.
such as \emph{groundedness} \citep{phd/Bogaerts15}.  %
Moreover, it is often easier to define a semantic operator, than to define the semantics from scratch.

Recently, AFT was generalized to also capture \emph{non-deterministic operators} \citep{nondetAFTarxiv} which allow for different options or choices in their output. 
A prime example of the occurrence of non-determinism in KRR is \emph{disjunctive logic programming}, and it was indeed shown that many semantics of disjunctive logic programming (specifically the weakly supported, (partial) stable, and well-founded semantics \citep{alcantara2005well}) are captured by non-deterministic AFT. 
In this paper, we make further contributions to the study of non-deterministic AFT, with a particular emphasis on disjunctive logic programs. 
On the one hand, (in Section \ref{section:ultimate}) we deepen the theory of non-deterministics AFT by investigating so-called \emph{ultimate semantics}. 
For standard AFT, \cite{denecker2002ultimate} have shown that with every two-valued operator, we can uniquely associate a most-precise approximator called the \emph{ultimate approximator}. 
When defining semantics of new formalisms, this even takes the need of defining an approximator away, since it suffices to define an \emph{exact} operator and its ultimate approximator comes for free.\footnote{However, ultimate semantics often come at the cost of increased computational complexity compared to their standard counterparts.}
Our first contribution is to show how ultimate approximations can be obtained for non-deterministic AFT, which we later illustrate using disjunctive logic programs with aggregates. This means we give the first constructive method for obtaining non-deterministic approximation operators. 
On the other hand, we also \emph{apply} non-deterministic AFT to two areas that have thus far been out of reach of AFT. 
In Section \ref{sec:seq}, we use it to define an algebraic generalisation of the \emph{semi-equilibrium semantics}, a semantics originally formulated for disjunctive logic programs \citep{amendola2016semi} but now, thanks to our results, available to any operator-based semantics. In Section \ref{sec:aggregates}, we apply the theory of non-deterministic AFT to disjunctive logic programs with \emph{aggregates} in the body, giving rise to a family of semantics for such programs.
\section{Background and Preliminaries}
\label{sec:back:prelim}

In this section, we recall disjunctive logic 
programming (Sec.~\ref{sec:LP}), approximation fixpoint theory for deterministic operators  (Sec.~\ref{sec:AFT}) and non-deterministic operators (Sec.~\ref{sec:non-det:aft}).

\subsection{Disjunctive Logic Programming}
\label{sec:LP}

In what follows we consider a propositional\footnote{For simplicity we restrict ourselves to the propositional case.} language ${\mathfrak L}$, 
whose atomic formulas are denoted by $p,q,r$ (possibly indexed), and that contains the propositional constants ${\sf T}$ (representing truth), ${\sf F}$ (falsity), 
${\sf U}$ (unknown), and ${\sf C}$ (contradictory information). The connectives in  ${\mathfrak L}$ include negation $\neg$, conjunction $\wedge$, disjunction $\vee$, 
and implication $\leftarrow$. Formulas are denoted by $\phi$, $\psi$, $\delta$ (again, possibly indexed). Logic programs in ${\mathfrak L}$ may be divided to different kinds as follows: a (propositional) {\em disjunctive logic program\/} ${\cal P}$ in ${\mathfrak L}$ (a dlp in short) is a finite set of rules of the form 
             $\bigvee_{i=1}^n p_i~\leftarrow~\psi$,  where the head $\bigvee_{i=1}^n p_i$ is a non-empty disjunction of atoms, and the body $\psi$
             is a formula not containing $\leftarrow$. A rule is called {\em normal\/} (nlp), if its body is a conjunction of literals (i.e., atomic formulas or negated atoms), and its head is atomic. A rule is \emph{disjunctively normal} if its body is a conjunction of literals and its head is a non-empty disjunction of atoms. 
We will use these denominations for programs if all rules in the program satisfy the denomination, e.g.\ a program is normal if all its rules are normal.             
The set of atoms occurring in $\mathcal{P}$ is denoted ${\cal A}_{\cal P}$. 

The semantics of dlps are given in terms of \emph{four-valued interpretations}. A {\em four-valued interpretation} of a program ${\cal P}$ is a pair $(x,y)$, where $x \subseteq {\cal A}_{\cal P}$ is the set of the atoms that are assigned a value in 
$\{{\sf T},{\sf C}\}$ and $y \subseteq {\cal A}_{\cal P}$ is the set of atoms assigned a value in $\{{\sf T},{\sf U}\}$. We define	 $-{\sf T}={\sf F}$, $-{\sf F}={\sf T}$ and ${\sf X}=-{\sf X}$ for ${\sf X}={\sf C},{\sf U}$.
 Truth assignments to complex formulas are as follows:
  {\setlength{\columnsep}{-25pt} 
\begin{multicols}{2}
    \begin{itemize}\item $(x,y)({p})=
\begin{cases}
      {\sf T} & \text{ if } {p} \in x \text{ and } {p} \in y,  \\
      {\sf U} & \text{ if } {p} \not\in x \text{ and }{p} \in y,  \\
      {\sf F} & \text{ if } {p} \not\in x \text{ and } {p} \not\in y,  \\
      {\sf C} & \text{ if } {p} \in x \text{ and } {p} \not\in y. 
\end{cases}$ \smallskip
\item $(x,y)(\lnot \phi)=- (x,y)(\phi)$,  
\item $(x,y)(\psi \land \phi)=lub_{\leq_t}\{(x,y)(\phi),(x,y)(\psi)\}$, 
\item $(x,y)(\psi \lor \phi)= glb_{\leq_t}\{(x,y)(\phi),(x,y)(\psi)\}$. 
\end{itemize}
    \end{multicols}}
A four-valued interpretation of the form $(x,x)$ may be associated with a {\em two-valued\/} (or {\em total\/}) interpretation $x$. 
$(x,y)$ is a {\em three-valued\/} 
(or {\em consistent\/}) interpretation, if $x \subseteq y$. 
Interpretations are compared by two order relations which form a pointwise extension of the structure ${\cal FOUR}$ consisting of ${\sf T}, {\sf F},{\sf C}$ and ${\sf U}$ with ${\sf U}<_i {\sf F},{\sf T}<_i {\sf C}$ and ${\sf F}<_t {\sf C},{\sf U}<_t {\sf T}$.
	The pointwise extension of these orders corresponds to the \emph{information order\/}, which is equivalently defined as $(x,y)\leq_i (w,z)$ iff $x\subseteq w$ and $z\subseteq y$, and  the \emph{truth order\/}, where $(x,y)\leq_t (w,z)$ iff $x\subseteq w$ and $y\subseteq z$. 

The immediate consequence operator for normal 
programs \citep{EmdenK76} is extended to dlp's as follows:

\begin{definition}[Immediate Consquence operator for dlp's]
\label{def:operator:disj:lp}
Given a dlp ${\cal P}$ and a two-valued interpretation $x$, we define:
(1) $\HR_{\cal P}(x)=\{\Delta\mid \bigvee\!\Delta\leftarrow \psi \in{\cal P} \text{ and } (x,x)(\psi) =  {\sf T}\}$; and (2)
 $\IC_{\cal P}(x)=\{y\subseteq \:\bigcup\!\HR_{\cal P}(x)  \mid \forall \Delta \in \HR_{\cal P}(x), \ y \cap \Delta \neq \emptyset \}$. 
 \end{definition}
Thus, $\IC_{\cal P}(x)$ consists of sets of atoms that occur in activated rule heads, each set contains at least one representative from every disjuncts of a rule in ${\cal P}$ whose body is $x$-satisfied.
 Denoting by $\wp({\cal S})$ the powerset of ${\cal S}$, $\IC_{\cal P}$ is an operator on the lattice 
$\tup{\wp({\cal A}_{\cal P}),\subseteq}$.\footnote{The operator $\IC_{\cal P}$ is a generalization of the immediate consequence operator from \cite[Definition 3.3]{fernandez1995bottom}, 
where the minimal sets of atoms in  $IC_{\cal P}(x)$ are considered. However, this requirement of minimality is neither necessary nor desirable in the consequence operator \citep{nondetAFTarxiv}.}

 Given a dlp ${\cal P}$ a consistent interpretation $(x,y)$ is
  a \emph{(three--valued) model\/} of ${\cal P}$, if for every $\phi\leftarrow \psi \in {\cal P}$, $(x,y)(\phi)\geq_t (x,y)(\psi)$.
 The GL-transformation $\frac{\cal P}{(x,y)}$ of a disjunctively normal dlp ${\cal P}$ with respect to a consistent $(x,y)$, is the positive
program obtained by replacing in every rule in ${\cal P}$ of the form
$p_1\lor\ldots\lor p_n \leftarrow \bigwedge_{i=1}^m q_i\land \bigwedge_{j=1}^n \lnot r_j$
a negated literal $\lnot r_i$ ($1\leq i\leq k$) by $(x,y)(\lnot r_i)$.
$(x,y)$ is a {\em three-valued stable model\/} of ${\cal P}$ iff it is a $\leq_t$-minimal model of $\frac{{\cal P}}{(x,y)}$.
\footnote{An overview of other semantics for dlp's can be found in previous work on non-deterministic AFT \citep{nondetAFTarxiv}.}

\subsection{Approximation Fixpoint Theory}
\label{sec:AFT}

We now recall basic notions from approximation fixpoint theory (AFT), as described by Denecker, Marek and Truszczy{\'n}ski~(\citeyear{denecker2000approximations}).  We restrict ourselves here to the necessary formal details, and refer to more detailed introductions by  Denecker, Marek and Truszczy{\'n}ski~(\citeyear{denecker2000approximations}) and Bogaerts (\citeyear{phd/Bogaerts15}) for more informal details.
AFT introduces constructive techniques for approximating the fixpoints of an operator $O$ over a lattice $L= \tup{{\cal L},\leq}$.\footnote{Recall that a lattice is a partially ordered set in which every pair of elements has a least upper bound and greatest lower bound denoted by $\sqcup$ and $\sqcap$, respectively. If every set of elements has a least upper bound and greatest lower bound, we call the lattice complete.}
 Approximations are pairs of elements $(x,y)$. Thus, given a lattice $L = \tup{{\cal L},\leq}$, the induced {\em bilattice\/} is the structure $L^2 =\tup{{\cal L}^2,\leq_i,\leq_t}$, in which
${\cal L}^2 = {\cal L} \times {\cal L}$, and for every $x_1,y_1,x_2,y_2 \in {\cal L }$, $(x_1,y_1) \leq_i (x_2,y_2)$ if $x_1 \leq x_2$ and $y_1 \geq y_2$, and $(x_1,y_1) \leq_t (x_2,y_2)$ if $x_1 \leq x_2$ and $y_1 \leq y_2$.\footnote{Note that we use small letters to denote elements of lattice, capital letters to denote 
sets of elements, and capital calligraphic letters to denote sets of sets of elements.}

An {\em approximating operator\/} ${\cal O}:{\cal L}^2\rightarrow {\cal L}^2$ of an operator $O:{\cal L}\rightarrow {\cal L}$ is an operator that maps every 
approximation $(x,y)$ of an element $z$ to an approximation $(x',y')$ of another element $O(z)$, thus approximating the behavior of the approximated operator $O$. 
In more details, an operator ${\cal O}:{\cal L}^2\rightarrow {\cal L}^2$ is {\em $\leq_i$-monotonic\/}, if when $(x_1,y_1)\leq_i(x_2,y_2)$, also ${\cal O}(x_1,y_1)\leq_i {\cal O}(x_2,y_2)$; 
           ${\cal O}$ is \emph{approximating\/}, if it is $\leq_i$-monotonic and for any $x\in {\cal L}$, ${\cal O}_l(x,x) = {\cal O}_u(x,x)$.\footnote
           {In some papers (e.g.,~\cite{denecker2000approximations}), an approximation operator is defined as a symmetric $\leq_i$-monotonic operator, 
           i.e.\ a $\leq_i$-monotonic operator s.t.\ for every $x,y\in {\cal L}$, ${\cal O}(x,y)=({\cal O}_l(x,y),{\cal O}_l(y,x))$ for some 
           ${\cal O}_l:{\cal L}^2\rightarrow {\cal L}$. However, the weaker condition we take here (taken from \cite{denecker2002ultimate}) is actually 
           sufficient for most results on AFT. \label{footnote:symetry} }
${\cal O}$ {\em approximates\/} of $O:{\cal L}\rightarrow {\cal L}$, if it is $\leq_i$-monotonic and ${\cal O}(x,x) = (O(x), O(x))$ (for every $x\in {\cal L}$).
Finally,
for a complete lattice $L $,  let ${\cal O}:{\cal L}^2 \rightarrow {\cal L}^2$ be an approximating operator.
We denote: ${\cal O}_{l}(\cdot,y) = \lambda x.{\cal O}_{l}(x,y)$ and similarly for ${\cal O}_{u}$.
The \emph{stable operator for ${\cal O}$\/} is then defined as $S({\cal O})(x,y)=(\lfp({\cal O}_l(.,y)),\lfp({\cal O}_u(x,.))$, where $\lfp(O)$ denotes the least fixpoint of an operator $O$.

Approximating operators induce a family of \emph{fixpoint semantics}. Given a complete lattice $L = \tup{{\cal L},\leq}$ and an approximating operator ${\cal O}:{\cal L}^2 \rightarrow {\cal L}^2$, $(x,y)$ is
  a \emph{Kripke-Kleene fixpoint\/} of ${\cal O}$ if $(x,y) =\lfp_{\leq_i}({\cal O}(x,y))$; $(x,y)$ is a \emph{three-valued stable fixpoint\/} of ${\cal O}$ if $(x,y)= S({\cal O})(x,y) $;  $(x,y)$ is a \emph{two-valued stable fixpoints\/} of ${\cal O}$ if $x=y$ and $(x,x)= S({\cal O})(x,x)$; $(x,y)$ is  the \emph{well-founded  fixpoint\/} of ${\cal O}$ if it is the $\leq_i$-minimal (three-valued) stable fixpoint of ${\cal O}$.

\subsection{Non-deterministic approximation fixpoint theory}\label{sec:non-det:aft}
AFT was generalized to non-deterministic operators, i.e.\ operators which map elements of a lattice to a set of elements of that lattice (like the operator $\IC_{\cal P}$ for DLPs) by \cite{nondetAFTarxiv}. We recall the necessary details, referring to the original paper for more details and explanations. 

A {\em non-deterministic operator on ${\cal L}$} is a function $O : {\cal L}\rightarrow \wp({\cal L}) \setminus \{\emptyset\}$. 
 For example, the operator $\IC_{\cal P}$ from Definition~\ref{def:operator:disj:lp} is a non-deterministic operator on the lattice $\tup{\wp({\cal A}_{\cal P}),\subseteq}$.

As the ranges of non-deterministic operators are {\em sets\/} of lattice elements, one needs a way to compare them, such as the {\em Smyth order\/} and the {\em Hoare order\/}.
Let $L = \tup{{\cal L},\leq}$ be a lattice, and let $X,Y \in \wp({\cal L})$. Then: $X \preceq^S_L Y$ if for every $y\in Y$ there is an $x\in X$ such that $x\leq y$; and $X \preceq^H_L Y$ if for every $x\in X$ there is a $y\in Y$ such that $x\leq y$.              Given some $X_1,X_2,Y_1,Y_2\subseteq {\cal L}$, $X_1\times Y_1 \preceq^A_i X_2\times Y_2$ iff $X_1\preceq^S_L X_2$ and $Y_2\preceq^H_L Y_1$. 
Let $L=\tup{{\cal L},\leq}$ be a lattice.
Given an operator ${\cal O}:{\cal L}^2\rightarrow {\cal L}^2$, we denote by ${\cal O}_l$  the operator defined by ${\cal O}_l(x,y)={\cal O}(x,y)_1$, and similarly for ${\cal O}_u(x,y)={\cal O}(x,y)_2$.
 An operator ${\cal O}:{\cal L}^2\rightarrow \wp({\cal L}){\setminus\emptyset}\times \wp({\cal L}){\setminus\emptyset}$ 
is called a {\em non-deterministic approximating operator\/} (ndao, for short), if it is $\preceq^A_i$-monotonic (i.e.\ $(x_1,y_1)\leq_i (x_2,y_2)$ implies ${\cal O}(x_1,y_1)\preceq^A_i {\cal O}(x_2,y_2)$),
and  is \emph{exact} (i.e., for every $x\in {\cal L}$, ${\cal O}(x,x)={\cal O}_l(x,x)\times {\cal O}_{l}(x,x)$). We restrict ourselves to ndaos ranging over consistent pairs $(x,y)$.

We finally define the stable operator (given an ndao ${\cal O}$) as follows.
The \emph{complete lower stable operator\/} is defined by (for any $y\in {\cal L}$) $C({\cal O}_l)(y) \ = 
 \{x \in {\cal L}\mid x\in {\cal O}_l(x,y) \mbox{ and }\lnot \exists x'< x: x'\in {\cal O}_l(x',y)\}$.
The \emph{complete upper stable operator\/} is defined by  (for any $x\in {\cal L}$) $C({\cal O}_u)(x) \ =\  \{y \in {\cal L}\mid y\in {\cal O}_u(x,y) \mbox{ and }\lnot \exists y'<y:  y' \in {\cal O}_u(x,y')\}$.
 The \emph{stable operator\/}: $S({\cal O})(x,y)=C({\cal O}_l)(y)\times C({\cal O}_u)(x)$. $(x,y)$ is a \emph{stable fixpoint\/} of ${\cal O}$ if $(x,y)\in S({\cal O})(x,y)$.\footnote{Notice that we slightly abuse notation and write $(x,y)\in S({\cal O})(x,y)$ 
to abbreviate $x\in  (S({\cal O})(x,y))_1$ and $y\in ( S({\cal O})(x,y))_2$, i.e.\ $x$ is a lower bound generated by $ S({\cal O})(x,y)$ and $y$ is an upper bound generated by $ S({\cal O})(x,y)$.}

Other semantics, e.g.\ the well-founded state and the Kripke-Kleene fixpoints and state are defined by Heyninck et al (\citeyear{nondetAFTarxiv}) and can be immediately obtained once an ndao is formulated. 

\begin{example}
\label{example:operator:disj:lp-2-b}
An example of an ndao approximating $\IC_{\cal P}$ (Definition \ref{def:operator:disj:lp}) is defined as follows (given a dlp ${\cal P}$ and an interpretation $(x,y)$):
 $\HRc^l_{\cal P}(x,y) = \{ \Delta \mid \bigvee\!\Delta \leftarrow \phi\in {\cal P}, (x,y)(\phi)\geq_t {\sf C}\}$, 
   $\HRc^u_{\cal P}(x,y) = \{ \Delta \mid \bigvee\!\Delta \leftarrow \phi\in {\cal P}, (x,y)(\phi)\geq_t {\sf U}\}$, 
   ${\cal IC}^\dagger_{\cal P}(x,y)=\{x_1\subseteq \bigcup\HRc^\dagger_{\cal P}(x,y) \mid \forall \Delta\in \HRc^\dagger_{\cal P}(x,y), \ x_1 \cap \Delta \neq \emptyset \}$ (for $\dagger\in \{l,u\}$), and
   ${\cal IC}_{\cal P}(x,y)=({\cal IC}^l_{\cal P}(x,y), {\cal IC}^u_{\cal P}(x,y))$. 

Consider the following dlp: ${\cal P}=\{ p\lor q\leftarrow\lnot q\}$.
The operator ${\cal IC}^l_{\cal P}$ behaves as follows:
\begin{itemize}
\item For any interpretation $(x,y)$ for which $q\in x$, $\HRc^l_{\cal P}(x,y)=\emptyset$ and thus ${\cal IC}^l_{\cal P}(x,y)=\{\emptyset\}$.
\item For any interpretation $(x,y)$ for which $q\not\in x$, $\HRc^l_{\cal P}(x,y)=\{\{p,q\}\}$ and thus ${\cal IC}^l_{\cal P}(x,y)=\{\{p\},\{q\},\{p,q\}\}$.
\end{itemize}
 Since ${\cal IC}_{\cal P}^l(x,y)={\cal IC}_{\cal P}^u(y,x)$ (see \cite[{Lemma 1}]{nondetAFTarxiv}), ${\cal IC}_{\cal P}$ behaves as follows:
\begin{itemize}
\item For any $(x,y)$ with $q\not\in x$ and $q\not\in y$, ${\cal IC}_{\cal P}(x,y)=\{\{p\},\{q\},\{p,q\}\}\times \{\{p\},\{q\},\{p,q\}\}$,
\item For any $(x,y)$ with $q\not\in x$ and $q\in y$, ${\cal IC}_{\cal P}(x,y)=\{\emptyset\}\times \{\{\{p\},\{q\},\{p,q\}\}$,
\item For any $(x,y)$ with $q\in x$ and $q\not\in y$, ${\cal IC}_{\cal P}(x,y)=\{\{p\},\{q\},\{p,q\}\}\times \{\emptyset\}$, and
\item For any $(x,y)$ with $q\in x$ and $q\in y$, ${\cal IC}_{\cal P}(x,y)=\{(\emptyset,\emptyset)\}$.
\end{itemize}
We see e.g.\ that $C({\cal IC}^l_{\cal P})(\{p\})=\{\{p\},\{q\}\}$ and thus $(\{p\},\{p\})$ is a stable fixpoint of ${\cal IC}_{\cal P}$. $(\emptyset,\{q\})$ is the second stable fixpoint of ${\cal IC}_{\cal P}$. $(\emptyset,\{p,q\})$ is  a fixpoint of ${\cal IC}_{\cal P}$ that is not stable.

In general, (total) stable fixpoints of ${\cal IC}_{\cal P}$ correspond to (total) stable models of ${\cal P}$, and weakly supported models of ${\cal IC}_{\cal P}$ correspond to fixpoints of ${\cal IC}_{\cal P}$. \citep{nondetAFTarxiv}.
\end{example}

\section{Ultimate Operators}\label{section:ultimate}
Approximation fixpoint theory assumes an approximation operator, but does not specify how to construct it. In the literature, one finds various ways to construct a deterministic approximation operator ${\cal O}$ that approximates a deterministic operator $O$. Of particular interest is the \emph{ultimate} operator \citep{denecker2002ultimate}, which is the \emph{most precise} approximation operator. In this section, we show that non-deterministic approximation fixpoint theory admits an ultimate operator, which is, however, different from the ultimate operator for deterministic AFT.

We first recall that for a \emph{deterministic} operator $O: {\cal L}\rightarrow {\cal L}$, the ultimate approximation ${\cal O}^u$ is defined by \cite{denecker2002ultimate} as follows:\footnote{We use the abbreviation ${\sf DMT^d}$ for \emph{deterministic} Denecker, Marek and Truszczy{\'n}ski to denote this operator, 
as to not overburden the use of ${\cal IC}^{\cal U}_{\cal P}$. Indeed, we will later see that the ultimate operator for non-disjunctive logic programs generalizes to an 
ndao that is different from the ultimate non-deterministic operator ${\cal IC}^{\cal U}_{\cal P}$.}
\[{\cal O}^{\sf DMT^d}(x,y)=(\sqcap O[x,y], \sqcup O[x,y])\footnote{Recall that denotes $\sqcap X$ the greatest lower bound of $X$ and $\sqcup X$ denotes the least upper bound of $X$.}\]
Where $O[x,y]:=\{O(z)\mid x\leq z\leq y\}$. 
This operator is shown to be the most precise operator approximating an operator $O$ \citep{denecker2002ultimate}. In more detail, for any (deterministic) approximation operator ${\cal O}$ approximating $O$, and any consistent $(x,y)$, ${\cal O}(x,y)<_i {\cal O}^{\sf DMT^d}(x,y)$. 

 The ultimate approximator for $\IC_{\cal P}$ for non-disjunctive logic programs ${\cal P}$ looks as follows:
\begin{definition}%
\label{def:ultimate:operator:DMT}
Given a normal logic program ${\cal P}$, we let:
${\cal IC}_{\cal P}^{{\sf DMT^d}}(x,y)=({\cal IC}_{\cal P}^{{\sf DMT^d},l}(x,y),\:{\cal IC}_{\cal P}^{{\sf DM^d},u}(x,y))$ 
with: ${\cal IC}_{\cal P}^{{\sf DMT^d},l}(x,y)=\bigcap_{x\subseteq z\subseteq y} \{\alpha\mid \alpha\leftarrow \phi\in {\cal P} \mbox{ and }z(\phi)={\sf T} \}$, and ${\cal IC}_{\cal P}^{{\sf DMT^d},u}(x,y)=\bigcup_{x\subseteq z\subseteq y} \{\alpha\mid \alpha\leftarrow \phi\in {\cal P} \mbox{ and }z(\phi)={\sf T}\}$.

\end{definition}

In this section, we define the ultimate semantics for the non-deterministic operators. 
In more detail, we constructively define an approximation operator that is most precise and has non-empty upper and lower bounds. Its construction is based on the following idea: we are looking for an operator ${\cal O}^{\cal U}$ s.t.\ for any ndao ${\cal O}$ that approximates $O$, ${\cal O}_l(x,y)\preceq^S_L {\cal O}_l^{\cal U}(x,y)$ (and similarly for the upper bound). As we know that ${\cal O}_l(x,y)\preceq^S_L O(z)$ for any $x\leq z\leq y$, we can obtain ${\cal O}_l^{\cal U}$ by simply gathering all applications of $O$ to elements of the interval $[x,y]$ i.e.\ we define:
\[ {\cal O}^{{\cal U}}_l(x,y)=\bigcup_{x\leq z\leq y} O(z)\]

The upper bound can be defined in the same way as the lower bound. Altogether, we obtain:
\[{\cal O}^{\cal U}(x,y)= {\cal O}^{\cal U}_l(x,y)\times {\cal O}^{\cal U}_l(x,y)\]

The following example illustrates this definition for normal logic programs:
\begin{example}\label{example:ultimate:1}
Let ${\cal P}=\{q\leftarrow \lnot p; p\leftarrow p\}$. Then $IC_{\cal P}(\emptyset)=IC_{\cal P}(\{q\})=\{q\}$ and $IC_{\cal P}(\{p\})=IC_{\cal P}(\{p,q\})=\{p\}$. Therefore, ${\cal IC}^{\cal U}_{\cal P}(\emptyset,\{p,q\})=\{\{p\},\{q\}\}\times \{\{p\},\{q\}\}$ whereas ${\cal IC}^{\cal U}_{\cal P}(\emptyset,\{q\})=\{\{q\}\}\times\{\{q\}\}$.
\end{example}

The ultimate approximation is the most precise ndao approximating the operator $O$:
\begin{proposition}\label{proposition:ultimate:is:most:precise}
Let a non-deterministic operator $O$ over a lattice $\langle {\cal L},\leq\rangle$ be given. Then ${\cal O}^{\cal U}$ is an ndao that approximates $O$. Furthermore, for any ndao ${\cal O}$ that approximates $O$ and for every $x,y\in{\cal L}$ s.t.\ $x\leq y$, it holds that ${\cal O}(x,y)\preceq^A_i{\cal O}^{\cal U}(x,y)$.
\end{proposition}

In conclusion, non-deterministic AFT admits, just like deterministic AFT, an ultimate approximation. However, as we will see in the rest of this section, the ultimate non-deterministic approximation operator ${\cal O}^{\cal U}$ does \emph{not} generalize the deterministic ultimate approximation operator defined by Denecker et al (\citeyear{denecker2002ultimate}). In more detail, we compare the non-deterministic ultimate operator ${\cal IC}^{\cal U}_{\cal P}$ with the deterministic ultimate ${\cal IC}^{\sf DMT}_{\cal P}$ from Definition \ref{def:ultimate:operator:DMT}.
Somewhat surprisingly, even when looking at normal logic programs, the operator ${\cal IC}^{\sf DMT^d}_{\cal P}$ does not coincide with the ultimate ndao  ${\cal IC}^{\cal U}_{\cal P}$ (and thus ${\cal IC}^{\sf DMT^d}_{\cal P}$ is \emph{not} the most precise ndao, even for non-disjunctive programs). The intuitive reason is that the additional expressivity of non-deterministic operators, which are not restricted to single lower and upper bounds in their outputs, allows to more precisely capture what is derivable in the ``input interval'' $(x,y)$.
\begin{example}[Example \ref{example:ultimate:1} continued]
Consider again ${\cal P}=\{q\leftarrow \lnot p; p\leftarrow p\}$. Applying the ${\sf DMT^d}$-operator gives: 
${\cal IC}_{\cal P}^{{\sf DMT^d}}(\emptyset,\{p,q\})=(\emptyset,\{p,q\})$.
Intuitively, the ultimate semantics ${\cal IC}_{\cal P}^{\cal U}(\emptyset,\{p,q\})=\{\{p\},\{q\}\}\times \{\{p\},\{q\}\}$ gives us the extra information  that we will always either derive $p$ or $q$, which is information a deterministic approximator can simply not capture.
Such a ``choice'' is not expressible within a single interval, hence the deterministic ultimate approximation is $(\emptyset,\{p,q\})$.
 This example also illustrates the fact that, when applying the ultimate ndao-construction to (non-constant) deterministic operators $O$, ${\cal O}^{\cal U}$ might be a \emph{non}-deterministic approximation operator.
\end{example}

However, one can still generalize the operator ${\cal IC}^{\sf DMT^d}_{\cal P}$ to disjunctive logic programs. 
We first generalize the idea behind ${\cal IC}_{\cal P}^{{\sf DMT^d},l}$ to an operator gathering the heads 
of rules that are true in every interpretation $z$ in the interval $[x,y]$. Similarly, ${\cal IC}_{\cal P}^{{\sf DMT^d},u}$ is generalized by gathering the heads of rules with bodies that are true in at least one interpretation in $[x,y]$:
\[{\cal HD}^{{\sf DMT},l}_{\cal P}(x,y)=\bigcap_{x\subseteq z\subseteq y} \HR_{\cal P}(z)\quad \mbox{ and }\quad {\cal HD}^{{\sf DMT},u}_{\cal P}(x,y)=\bigcup_{x\subseteq z\subseteq y} \HR_{\cal P}(z)\}.\] 
The upper and lower immediate consequences operator are then straightforwardly defined, that is: by taking all interpretations that only contain atoms in ${\cal HD}^{{\sf DMT},\dagger}_{\cal P}(x,y)$ and contain at least one member of every head $\Delta\in {\cal HD}^{{\sf DMT},\dagger}_{\cal P}(x,y)$ (for $\dagger\in \{u,l\}$):
\[{\cal IC}_{\cal P}^{{\sf DMT},\dagger}(x,y)=\{z\subseteq \bigcup{\cal HD}^{{\sf DMT},\dagger}_{\cal P}(x,y) \mid \forall \Delta\in{\cal HD}^{{\sf DMT},\dagger}_{\cal P}(x,y)\neq\emptyset: z\cap \Delta\neq \emptyset\}.\] 
Finally, the ${\sf DMT}$-ndao is defined as:
${\cal IC}_{\cal P}^{{\sf DMT}}(x,y)={\cal IC}_{\cal P}^{{\sf DMT},l}(x,y)\times {\cal IC}_{\cal P}^{{\sf DMT},u}(x,y)$. We have:

\begin{proposition}[{\cite[Proposition 3]{nondetAFTarxiv}}]
For any disjunctive logic program ${\cal P}$, ${\cal IC}_{\cal P}^{{\sf DMT}}$ is an ndao that approximates $IC_{\cal P}$.
\end{proposition}
Notice that for a non-disjunctive program ${\cal P}$, $\bigcup{\cal IC}_{\cal P}^{{\sf DMT},\dagger}(x,y)= \bigcup{\cal HD}^{{\sf DMT},\dagger}_{\cal P}(x,y)={\cal IC}_{\cal P}^{{\sf DMT^d},\dagger}(x,y)$ (for $\dagger\in \{u,l\}$), i.e.\ the non-deterministic version reduces to the deterministic version when looking at non-disjunctive programs.
Notice furthermore the operators ${\cal HD}^{{\sf DMT},l}_{\cal P}(x,y)$ and ${\cal HD}^{{\sf DMT},u}_{\cal P}(x,y)$ are only defined for consistent interpretations $(x,y)$. We leave the extension of this operator to inconsistent interpretations for future work. 

\begin{example}
Consider  again the program ${\cal P}=\{p\lor q\leftarrow \lnot q\}$ from Example~\ref{example:operator:disj:lp-2-b}. ${\cal IC}^{{\sf DMT},l}_{\cal P}$ behaves as follows:
\begin{itemize}
\item If $q\in y$ then ${\cal HD}^{{\sf DMT},l}_{\cal P}(x,y)=\emptyset$ and thus ${\cal IC}^{{\sf DMT},l}_{\cal P}(x,y)=\emptyset$.
\item If $q\not\in y$ then ${\cal HD}^{{\sf DMT},l}_{\cal P}(x,y)=\{\{p,q\}\}$ and ${\cal IC}^{{\sf DMT},l}_{\cal P}(x,y)=\{\{p\},\{q\},\{p,q\}\}$.
\end{itemize}
 ${\cal IC}^{{\sf DMT},u}_{\cal P}$ behaves as follows:
\begin{itemize}
\item If $q\in x$ then ${\cal HD}^{{\sf DMT},u}_{\cal P}(x,y)=\emptyset$ and thus ${\cal IC}^{{\sf DMT},u}_{\cal P}(x,y)=\emptyset$.
\item If $q\not\in x$ then ${\cal HD}^{{\sf DMT},u}_{\cal P}(x,y)=\{\{p,q\}\}$ and thus ${\cal IC}^{{\sf DMT},u}_{\cal P}(x,y)=\{\{p\},\{q\},\{p,q\}\}$.
\end{itemize}
Thus e.g.\ ${\cal IC}^{\sf DMT}_{\cal P}(\emptyset,\{p,q\})= \{\emptyset\}\times \{\{p\},\{q\},\{p,q\}\}$ and ${\cal IC}^{\sf DMT}_{\cal P}(\{p\},\{p\})=\{\{p\},\{q\},\{p,q\}\}\times \{\{p\},\{q\},\{p,q\}\}$. We thus see that $(\{p\},\{p\})$ is a stable fixpoint of ${\cal IC}^{\sf DMT}_{\cal P}$.

A slightly extended program ${\cal P}=\{q\leftarrow \lnot q; p\lor q\leftarrow  q\}$ shows some particular but unavoidable behavior of this operator. ${\cal IC}^{{\sf DMT},l}_{\cal P}(\emptyset,\{q\})=\{\emptyset\}$ as $\HR_{\cal P}(\emptyset)=\{\{q\}\}$ and $\HR_{\cal P}(\{q\})=\{\{p,q\}\}$. Note that the lower bound for is \emph{not} the stronger $\{p\}$. This would result in a loss of $\preceq^A_i$-monotonicity, as the lower bound $\{\{q\}\}$ for the less informative $(\emptyset,\{q\})$ would be $\preceq^S_L$-incomparable to the lower bound $\{\{p\},\{q\},\{p,q\}\}$ of the more informative $(\{q\},\{q\})$.
\end{example}
We have shown in this section that non-deterministic AFT admits an ultimate operator, thus providing a way to construct an ndao based on a non-deterministic operator. We have also shown that the ultimate ndao diverges from the ultimate operator for deterministic AFT, but that this deterministic ultimate operator can be generalized to disjunctive logic programs. Both operators will be used in Section \ref{sec:aggregates} to define semantics for DLP's with aggregates.

\section{Semi-Equilibrium Semantics}\label{sec:seq}
To further extend the reach of non-deterministic AFT, we generalize yet another semantics for dlp's, namely the \emph{semi-equilibrium semantics} \citep{amendola2016semi}. The semi-equilibrium semantics is a semantics for disjunctive logic programs that has been studied for disjunctively normal logic programs. This semantics is a three-valued semantics that fulfills the following properties deemed desirable by \citet{amendola2016semi}: 
(1) Every (total) answer set of ${\cal P}$ corresponds to a semi-equilibrium model;
(2) If ${\cal P}$ has a (total) answer set, then all of its semi-equilibrium models are (total) answer sets;
(3) If ${\cal P}$ has a classical model, then ${\cal P}$ has a semi-equilibrium model.
We notice that these conditions can be seen as a view on approximation of the total stable interpretations alternative to the well-founded semantics. We do not aim to have the last word on which semantics is the most intuitive or desirable. Instead, we will show here that semi-equilibrium models can be represented algebraically, and thus can be captured within approximation fixpoint theory. This leaves the choice of exact semantics to the user once an ndao has been defined, and allows the use of the semi-equilibrium semantics for formalisms other than nlps, such as disjunctive logic programs with aggregates (see below) or conditional ADFs.

Semi-equilibrium models are based on the \emph{logic of here-and-there} \citep{pearce2006equilibrium}. An \emph{HT-interpretation} is a pair $(x,y)$ where $x\subseteq y$ (i.e.\ a consistent pair in AFT-terminology). Satisfaction of a formula $\phi$, denoted $\modelsht$, is defined recursively as follows:
\begin{itemize}
\item $(x,y)\modelsht \alpha$ if $\alpha\in x$ for any $\alpha\in {\cal A}_{\cal P}$,
\item $(x,y)\modelsht \lnot \phi$ if $(y,y)(\phi)\neq{\sf T}$, and $(x,y)\not\modelsht \bot$,
\item $(x,y)\modelsht \phi\land[\lor]\psi$ if $(x,y)\modelsht \phi$ and [or] $(x,y)\modelsht \psi$, 
\item $(x,y)\modelsht \phi\rightarrow \psi$ if (a) $(x,y)\not\modelsht\phi$ or $(x,y)\modelsht\psi$, and (b) $(y,y)(\lnot \phi\lor\psi)={\sf T}$.
\end{itemize}
The ${\sf HT}$-models of ${\cal P}$ are defined as ${\sf HT}({\cal P})=\{(x,y)\mid \forall\psi\leftarrow \phi\in {\cal P}: (x,y)\modelsht\phi\rightarrow \psi\}$.

Semi-equilibrium models are a special class of ${\sf HT}$-models. They are obtained by performing two minimization steps on the set of ${\sf HT}$-models of a program. The first step is obtained by minimizing w.r.t.\ $\leq_t$.\footnote{\cite{amendola2016semi} proceeds as follows. First, ${\sf HT}^\kappa({\cal P})=\{x\cup \{{\sf K}\alpha\mid\alpha\in y\}\}$ is constructed, and then the $\subseteq$-minimal sets in ${\sf HT}^\kappa({\cal P})$ are selected. It is straightforward to see that this is equivalent to minimizing the original interpretations w.r.t.\ $\leq_t$.} The second step is obtained by selecting the \emph{maximal canonical models}. For this, the \emph{gap} of an interpretation is defined as $gap(x,y)=y\setminus x$,\footnote{Again, \cite{amendola2016semi} proceeds in a slightly more convoluted way by defining $gap(I)=\{{\sf K}\alpha\in I\mid \alpha\not\in I\}$ for any $I\in {\sf HT}^\kappa({\cal P})$.} and, for any set of interpretations $\Xb$, the \emph{maximally canonical interpretations} are $mc(\Xb)=\{(x,y)\in \Xb\mid {\not\exists} (w,z)\in \Xb: gap(x,y)\supset gap(w,z)\}$. The semi-equilibrium models of ${\cal P}$ are then defined as:
$ \mathcal {{SEQ}}({\cal P})= mc\left( \min_{\leq_t}({\sf HT}({\cal P})\right)$.
\begin{example}
We illustrate these semantics with the program ${\cal P}=\{p\leftarrow \lnot p, s\lor q\leftarrow \lnot s, s\lor q\leftarrow \lnot q\}$. 
Then ${\sf HT}({\cal P})=\{(x,y)\mid \{p\}\subseteq y\subseteq \{p,q,s\}, x\subseteq y, \{q,s\}\cap y\neq \emptyset\}$. Furthermore, $\min_{\leq_t} ({\sf HT}({\cal P}))=\{(\emptyset,\{p,q,s\}), (\{q\},\{q,p\}), (\{s\},\{s,p\})\}$. As $gap(\emptyset,\{p,q,s\})=\{p,q,s\}$ and $gap (\{q\},\{q,p\})=gap  (\{s\},\{s,p\})=\{p\}$, $\mathcal{SEQ}=\{ (\{q\},\{q,p\}), (\{s\},\{s,p\})\}$.
\end{example}
Before we capture the ideas behind this semantics algebraically, we look a bit deeper into the relationship between ${\sf HT}({\cal P})$-models and the classical notion of three-valued models of a program (see Section \ref{sec:LP}). 
We first observe that ${\sf HT}$-models of a program are a proper superset of the three-valued models of a program:
\begin{proposition}\label{prop:three-valued:are:subset:of:HT}
Let a disjunctively normal logic program ${\cal P}$ and a consistent intepretation $(x,y)$ be given. Then if $(x,y)$ is a model of ${\cal P}$, it is an ${\sf HT}$-model of ${\cal P}$. However, not every ${\sf HT}$-model is a model of ${\cal P}$.
\end{proposition}

We now define the concept of a ${\sf HT}$-pair algebraically, inspired by  Truszczy\'nski (\citeyear{truszczynski2006strong}):

\begin{definition}
Given an ndao ${\cal O}$ approximating a non-determinstic operator $O$, a pair $(x,y)$ is a \emph{{\sf HT}-pair} (denoted $(x,y)\in {\sf HT}({\cal O})$) if the following three conditions are satisfied:
(1) $x\leq y$,
(2) $O(y)\preceq^S_L y$, and
(3) ${\cal O}_l(x,y)\preceq^S_L x$.
\end{definition}

This simple definition faithfully transposes the ideas behind {\sf HT}-models to an algebraic context. Indeed, applying it to ${\cal IC}_{\cal P}$ gives use exactly the {\sf HT}-models of ${\cal P}$:
\begin{proposition}\label{prop:ht:models:for:lp:are:faithful}
Let some normal disjunctive logic program ${\cal P}$ be given. Then:
${\sf HT}({\cal P})={\sf HT}({\cal IC}_{\cal P})$. 
\end{proposition}

We now show that  exact $\leq_t$-minimal ${\sf HT}$-models of ${\cal O}$ are stable interpretations of ${\cal O}$ in our algebraic setting. The opposite direction holds as well: total stable fixpoints are $\leq_t$-minimal ${\sf HT}$-pairs of ${\cal O}$. In fact, \emph{every} total fixpoint of ${\cal O}$ is a ${\sf HT}$-pair of ${\cal O}$. We assume that ${\cal O}$ is \emph{upwards coherent}, i.e.\ for every $x,y\in{\cal L}$, ${\cal O}_l(x,y)\preceq^S_L {\cal O}_u(x,y)$. In the appendix, we provide more details on upwards coherent operators. Notice that all ndaos in this paper are upwards coherent.
\begin{proposition}\label{prop:ht-models:and:stable:fixpoints}
Given an upwards coherent ndao ${\cal O}$, (1) if $(x,x)\in {\cal O}(x,x)$ then $(x,x)\in {\sf HT}({\cal O})$;
and (2) $(x,x)\in \min_{\leq_t}({\sf HT}({\cal O}))$ iff $(x,x)\in S({\cal O})(x,x)$.
\end{proposition}

The second concept that we have to generalize to an algebraic setting is that of maximal canonical models. Recall that $gap(x,y)$ consists of the atoms which are neither true nor false, i.e.\ it can be used as a measure of the informativeness or precision of a pair. For the algebraic generalization of this idea, it is useful to assume that the lattice under consideration admits a difference for every pair of elements.\footnote{If a lattice does not admit a difference for some elements, one cannot characterise the semi-equilibrium semantics exactly, but can still obtain an approximate characterisation. We detail this in the appendix.} In more detail, $z\in {\cal L}$ is the \emph{difference} of $y$ w.r.t.\ $x$ if $z\sqcap x=\bot$ and $x\sqcup y=x\sqcup z$.  If the difference is unique we denote it by $x\oslash y$. As an example, note that any Boolean lattice admits a unique difference for every pair of elements. We can then define $\mathsf{mc}(\mathbf{X})=\argmin_{(x,y)\in \mathbf{X}}\{y\oslash x\}$.  This allows us to algebraically formulate the semi-equilibrium models of an ndao ${\cal O}$ as \[{\cal SEQ}({\cal O})=\mathsf{mc}\left(\min_{\leq_t}( {\sf HT}({\cal O})  )\right)\]
The properties mentioned at the start of this section are preserved, and this definition generalizes the semi-equilibrium models for disjunctive logic programs by \cite{amendola2016semi}:
\begin{proposition}\label{prop:algebraic:seq:is:nice}
Let an upwards coherent ndao ${\cal O}$ over a finite lattice be given s.t.\ every pair of elements admits a unique difference.
Then ${\cal SEQ}({\cal O})\neq\emptyset$. 
Furthermore, if there is some $(x,x)\in \mathsf{mc}(\min_{\leq_t}( {\sf HT}({\cal O})  ))$ then ${\cal SEQ}({\cal O})=\{(x,x)\in {\cal L}^2\mid (x,x)\in S({\cal O})(x,x)\}$.
\end{proposition}
\begin{corollary}
Let a disjunctively normal logic program ${\cal P}$ be given. Then ${\cal SEQ}({\cal IC}_{\cal P})={\cal SEQ}({\cal P})$.
\end{corollary}

In this section, we have shown that semi-equilibrium models can be characterized algebraically.
This means semi-equilibrium models can now be obtained for other ndao's (e.g.\ those from Section \ref{sec:aggregates}, as illustrated in \ref{semi-equilibrium:to:aggregates}), thus greatly enlarging the reach of these semantics. 

We end this section by making a short, informal comparison between the semi-equilibrium models and the well-founded state for ndaos \citep{nondetAFTarxiv}. Both constructions have a similar goal: namely, approximate the (potentially non-existent) total stable interpretations. In the case of the semi-equilibrium models, the set of semi-equilibrium models coincides with the total stable interpretations if they exist, whereas the well-founded state approximates any stable interpretation (and thus in particular the total stable interpretations), but might not coincide with them. When it comes to existence, we have shown here that the semi-equilibrium models exist for any ndao, just like the well-founded state. Thus, the well-founded state and semi-equilibrium models seem to formalize two different notions of approximation. Which notion is most suitable is hard to decide \emph{in abstracto} but will depend on the exact application context.

\section{Application to DLPs with Aggregates}\label{sec:aggregates}
We apply non-deterministic AFT to disjunctive logic programs with aggregates by studying three ndaos: the ultimate, {\sf DMT} and the trivial operators. We show the latter two generalize the ultimate semantics \citep{pelov2007well} respectively  the semantics by \cite{gelfond2019vicious}. 

\subsection{Preliminaries on aggregates}\label{sec:agg:preliminaries}
We survey the necessary preliminaries on aggregates and the corresponding programs, restricting ourselves to propositional aggregates and leaving aggregates with variables for future work.

A set term $S$ is  a set of pairs of the form $[ \overline{t}: \conj]$ with $t$ a list of constants and $\conj$ a ground conjunction of standard atoms
For example, 
$[1:p; 2:q; -1:r]$ intuitively assigns $1$ to $p$, $2$ to $q$ and $-1$ to $r$. 
An \textit{aggregate function} is of the form $f(S)$ where $S$ is a set term, and $f$ is an \textit{aggregate function symbol} (e.g.\ $\#{\tt Sum}$, $\#{\tt Count}$ or $\#{\tt Max}$). An \emph{aggregate atom} is an expression of the form $f(S)\ast w$ where $f(S)$ is an aggregate function, $\ast\in \{<,\leq,\geq,>,=\}$ and $w$ is a numerical constant. We denote by   ${\sf At}(f(S)\ast w)$ the atoms occuring in $S$.

A \emph{disjunctively normal aggregate program} consists of rules of the form (where $\Delta$ is a set of propositional atoms, and $\alpha_1,\ldots,\alpha_n,\beta_1,\ldots,\beta_m$ are aggregate or propositional atoms):
\[ \bigvee \Delta\leftarrow \alpha_1,\ldots,\alpha_n,\lnot \beta_1,\ldots,\lnot \beta_m\]

An aggregate symbol is evaluated w.r.t.\ a set of atoms as follows. First, let $x(S)$ denote the multiset $[t_1 \mid \langle t_1,\ldots,t_n:{Conj}\rangle\in S\mbox{ and }{Conj}\mbox{ is true w.r.t. }x]$. $x(f(S))$ is then simply the result of the application of $f$ on $x(S)$. If the multiset $x(S)$ is not in the domain of $f$, $x(f(s))=\curlywedge$ where $\curlywedge$ is a fixed symbol not occuring in ${\cal P}$. An aggregate atom $f(S)\ast w$ is true w.r.t.\ $x$ (in symbols, $x(f\ast w)={\sf T}$) if: (1) $x(f(S))\neq \curlywedge$ and (2) $x(f(S))\ast w$ holds; otherwise, $f(S)\ast w$ is false (in symbols, $x(f\ast w)={\sf F}$). $\lnot f(S)\ast w$ is true if: (1) $x(f(S))\neq \curlywedge$ and (2) $x(f(S))\ast w$ does not hold; otherwise, $\lnot f(S)\ast w$ is false. Evaluating a conjunction of aggregate atoms is done as usual. We can now straightforwardly generalize the immediate consequence operator for disjunctive logic programs to disjunctive aggregate programs by generalizing $\HR_{\cal P}$ to take into account aggregate formulas as described above: $\HR_{\cal P}(x)=\{ \Delta\mid \bigvee\Delta\leftarrow \phi\in {\cal P}, x(\phi)={\sf T}\}$. $\IC_{\cal P}$ from Definition \ref{def:operator:disj:lp} is then generalized straightforwardly by simply using the generalized $\HR_{\cal P}$. 
Thus, the only difference with the immediate consequence operator for dlp's is that the set of activated heads $\HR_{\cal P}$ now takes into account the truth of aggregates as well.

The first semantics we consider is the one formulated by  \citet{gelfond2019vicious} (defined there only for logic programs with aggregates occurring positively in the body of a rule):
\begin{definition}
Let a disjunctively normal aggregate logic program ${\cal P}$ s.t.\ for every $\bigvee\Delta\leftarrow \bigwedge_{i=1}^n\alpha_i\land \bigwedge_{j=1}^m \lnot \beta_j\in {\cal P}$, $\beta_j$ is a normal (i.e.\ non-aggregate) atom. Then the {\sf GZ}-reduct of ${\cal P}$ w.r.t.\ $x$ is defined by doing, for every $r=\bigvee\Delta\leftarrow \bigwedge_{i=1}^n\alpha_i\land \bigwedge_{j=1}^m \lnot \beta_j\in {\cal P}$, the following:
(1) if an aggregate atom $\alpha_i$ is false or undefined for some $i=1,\ldots,n$, delete $r$;
(2) otherwise, replace every aggregate atom $\alpha_i=f(S)\ast w$ by $\bigcup\{Conj\mbox{ occurs in S}\mid x(Conj)={\sf T}\}$.
We denote the {\sf GZ}-reduct of ${\cal P}$ by ${\cal P}^x_{\sf GZ}$. Notice that this is a disjunctively normal logic program. 
A set of atoms $x\subseteq {\cal A}_{\cal P}$ is a \emph{{\sf GZ}-answer set of ${\cal P}$} if $(x,x)$ is an answer set of ${\cal P}^x_{\sf GZ}$.
\end{definition}

\begin{example}
Consider the program ${\cal P}=\{p\leftarrow \#{\tt Sum}[1: p,q]>0; p\leftarrow \#{\tt Sum}[1: q]>0; q\leftarrow \#{\tt Sum}[1: s]<1\}$. We check whether $\{p,q\}$ is a ${\sf GZ}$-answer set as follows:
\begin{enumerate}
\item The ${\sf GZ}$-reduct is
$ 
{\cal P}^{\{p,q\}}_{\sf GZ}=\{ p\leftarrow p,q;\quad p\leftarrow q;\quad q\leftarrow \}.
$
In more detail, as $\{p,q\}(\#{\tt Sum}[1: p,q]>0)={\sf T}$, we replace $\#{\tt Sum}[1: p,q]>0$ in the first rule by the atoms in the condition of this aggregate atom verified by $\{p,q\}$, namely $p$ and $q$. Similarly for the other rules.
\item As $\{p,q\}$ (or, to be formally more precise, $(\{p,q\},\{p,q\})$)  is a minimal model of $\frac{{\cal P}^{\{p,q\}}_{\sf GZ}}{(\{p,q\},\{p,q\})}$, we see $\{p,q\}$ is a {\sf GZ}-answer set of ${\cal P}$.
\end{enumerate}
\end{example}

We now move to the semantics by \cite{denecker2002ultimate}. They are defined only for non-disjunctive aggregate programs. They are defined on the basis of the ultimate (deterministic)  approximator ${\cal IC}^{\sf DMT}_{\cal P}$ (Definition \ref{def:ultimate:operator:DMT}). In more detail, an interpretation $(x,y)$ is \emph{${\sf DMT^d}$-stable} if and only if $(x,y)\in S({\cal IC}^{\sf DMT^d}_{\cal P})(x,y)$, i.e.\ $x\in \lfp({\cal IC}^{\sf DMT^d}_{\cal P}(.,y))$ and  $y\in \lfp({\cal IC}^{\sf DMT^d}_{\cal P}(x,.))$.

\begin{example}
Consider the program ${\cal P}=\{p\leftarrow \#{\tt Sum}[1:p]>0;\quad  p\leftarrow  \#{\tt Sum}[1:p]<1\}$.
$(\{p\},\{p\})$ is an ${\sf DMT^d}$-stable model of ${\cal P}$, but the program has no {\sf GZ}-stable models.

We first explain why $\{p\}$ is not a {\sf GZ}-stable model. First, we construct ${\cal P}^{\{p\}}_{\sf GZ}=\{ p\leftarrow p\}$. Since $\{p\}$ is not a stable model of ${\cal P}^{\{p\}}_{\sf GZ}$, we see that $\{p\}$ is not a {\sf GZ}-stable model.  Likewise, since ${\cal P}^{\emptyset}_{\sf GZ}=\{p\leftarrow \emptyset\}$, we see that $\emptyset$ is not a stable model of ${\cal P}^{\emptyset}_{\sf GZ}$ and therefore not {\sf GZ}-stable. 

To see $\{p\}$ is a ${\sf DMT^d}$-stable model, observe that ${\cal IC}_{\cal P}^{{\sf DMT^d},l}(\emptyset,\{p\})={\cal IC}_{\cal P}^{{\sf DMT^d},l}(\{p\},\{p\})=\{p\}$. Thus, $\lfp({\cal IC}_{\cal P}^{{\sf DMT^d},l}(.,\{p\})=\{p\}$, i.e.\ $(\{p\},\{p\})=S({\cal IC}^{\sf DMT^d}_{\cal P})(\{p\},\{p\})$. 
\end{example}

\subsection{Non-Deterministic Approximation Operators for Disjunctive Aggregate Programs}
We now proceed to define ndaos for disjunctive aggregate programs. The first ndao we consider generalizes the \emph{trivial} operator \citep{pelov2007well}, which maps two-valued interpretations to their immediate consequences whereas three-valued interpretations are mapped to the least precise pair $(\emptyset,{\cal A}_{\cal P})$ (or, in the non-deterministic case, $\{\emptyset\}\times \{{\cal A}_{\cal P}\}$). We also study the ndao ${\cal IC}^{\sf DMT}_{\cal P}$ based on the deterministic ultimate approximation, and the ultimate ndao ${\cal IC}_{\cal P}^{\cal U}$.
\begin{definition}
Given a disjunctively normal aggregate program ${\cal P}$ and a (consistent) interpretation $(x,y)$, let 
\begin{eqnarray*}
&{\cal IC}_{\cal P}^{\sf GZ}(x,y)&=
 \begin{cases}
 \IC_{\cal P}(x)\times \IC_{\cal P}(x) & \mbox{ if }x=y\\
\{\emptyset\}\times \{{\cal A}_{\cal P}\} & \mbox{ otherwise}
\end{cases}
\end{eqnarray*}
\end{definition}

The ndaos ${\cal IC}^{\sf DMT}_{\cal P}$ and ${\cal IC}^{\cal U}_{\cal P}$ are defined exactly the same as in section 3 (recall that $\IC_{\cal P}(x)$ was generalized for aggregates in Section \ref{sec:agg:preliminaries}).
We illustrate these semantics with an example:
\begin{example}\label{ex:running:aggregates}
Let ${\cal P}=\{r\lor q\leftarrow   \#{\tt Sum}[1:s]>0; s\leftarrow  \#{\tt Sum}[1:r, 1:q]>0\}$ be given.

We first look at ${\cal IC}_{\cal P}^{\sf GZ}$. As an example of a fixpoint, consider $(\{r,s\},\{r,s\})$. Notice first that $\#{\tt Sum}[1:r, 1:q]>0$ and $ \#{\tt Sum}[1:r, 1:q]>0$ are true in $\{r,s\}$. Thus, $\HR_{\cal P}=\{\{r,q\},\{s\}\}$ and ${\cal IC}^{{\sf GZ}}_{\cal P}(\{r,s\},\{r,s\})=\{\{r,s\},\{q,s\},\{r,q,s\}\}\times \{\{r,s\},\{q,s\},\{r,q,s\}\}$. 

We now look at the ${\sf DMT}$-semantics. For this, we first calculate $\HR_{\cal P}$ and $\IC_{\cal P}$ for all members of $\wp(\{r,q,s\})$ (with $\Delta_1=\{\{r\},\{q\},\{r,q\}\}$ and $\Delta_2=\{\{s,r\},\{s,q\},\{s,r,q\}\}$):

\begin{adjustbox}{width=0.95\textwidth,center}
\begin{oldtabular}{l|lllllllll} \hline\hline
$x$ & $\emptyset$ & $\{s\}$ & $\{q\}$ & $\{r\}$ & $\{r,q\}$ & $\{r,s\}$ & $\{q,s\}$ & $\{s,q,r\}$& \\ \hline
$\HR_{\cal P}(x)$ & $\emptyset$ & $ \{\{r,q\}\}$ & $\{\{s\}\}$ & $\{\{s\}\}$ & $\{\{s\}\}$ & $\{\{r,q\},\{s\}\}$& $\{\{r,q\},\{s\}\}$& $\{\{r,q\},\{s\}\}$& \\ \hline
$\IC_{\cal P}(x)$ & $\{\emptyset\}$ & $\Delta_1$ & $\{\{s\}\}$ &   $\{\{s\}\}$ & $\Delta_2$ &  $\Delta_2$ &  $\Delta_2$ &  $\Delta_2$& \\ \hline\hline
\end{oldtabular}
\end{adjustbox}

We then see that e.g.\ ${\cal IC}^{\sf DMT}_{\cal P}(\{r,s\},\{r,s\})=\{\{r,s\},\{q,s\},\{r,q,s\}\}\times \{\{r,s\},\{q,s\},\{r,q,s\}\}$ whereas ${\cal IC}^{\sf DMT}_{\cal P}(\emptyset,\{r,s\})=\{\emptyset\}\times \{ \{r,s\},\{q,s\},\{r,q,s\}\}$.

We see that ${\cal IC}^{\cal U}_{\cal P}(\{r,s\},\{r,s\})=\{\{r,s\},\{q,s\},\{r,q,s\}\}\times \{\{r,s\},\{q,s\},\{r,q,s\}\}$ whereas ${\cal IC}^{\cal U}_{\cal P}(\emptyset,\{r,s\})=\wp(\{r,s,q\})\times \wp(\{r,s,q\})$.
\end{example}

We now show that these operators are approximation operators with increasing orders of precision: ${\cal IC}^{\sf GZ}_{\cal P}$ is the least precise, ${\cal IC}^{\sf DMT}_{\cal P}$ holds a middle ground and ${\cal IC}^{\cal U}$ is the most precise:
\begin{proposition}\label{prop:aggregate:operators:are:ndaos}
Let some $\xi\in \{{\sf DMT},{\sf GZ},{\cal U}\}$ and a  disjunctively normal aggregate logic program ${\cal P}$ be given. Then ${\cal IC}_{\cal P}^{\xi}(x,y)$ is an ndao approximating $\IC_{\cal P}$. For any $(x,y)$, ${\cal IC}^{\sf GZ}_{\cal P}(x,y)\preceq^A_i {\cal IC}^{\sf DMT}_{\cal P}(x,y)\preceq^A_i {\cal IC}^{\cal U}_{\cal P}(x,y)$.
\end{proposition}

The following properties follow from the general properties shown by \cite{nondetAFTarxiv}:
\begin{proposition}
Let some $\xi\in \{{\sf DMT},{\sf GZ},{\cal U}\}$ and a  disjunctively normal aggregate logic program ${\cal P}$ be given. Then:
(1) $S({\cal IC}^{\epsilon}_{\cal P})(x,y)$ exists for any $x,y\subseteq {\cal A}_{\cal P}$, and
(2) every stable fixpoint of ${\cal IC}_{\cal P}^{\epsilon}$ is a $\leq_t$-minimal fixpoint of  ${\cal IC}_{\cal P}^{\epsilon}$.
\end{proposition}

The ndao ${\cal IC}_{\cal P}^{\sf GZ}$ only admits two-valued stable fixpoints, and these two-valued stable fixpoints generalize the {\sf GZ}-semantics \citep{gelfond2019vicious}:
\begin{proposition}\label{prop:gz:is:two-valued} \label{prop:GZ:generalizes:GZ}
If $(x,y)\in \min_{\leq_t}({\cal IC}^{\sf GZ}_{\cal P}(x,y))$ then $x=y$. 
Let a disjunctively normal aggregate aggregate logic program ${\cal P}$ s.t.\ for every $\bigvee\Delta\leftarrow \bigwedge_{i=1}^n\alpha_i\land \bigwedge_{j=1}^m \lnot \beta_j\in {\cal P}$, $\beta_i$ is a normal atom be given.
$(x,x)\in S({\cal IC}^{\sf GZ}_{\cal P})(x,x)$ iff $x$ is a {\sf GZ}-answer set of ${\cal P}$. 
\end{proposition}

We finally show that stable semantics based on ${\cal IC}^{\sf DMT}_{\cal P}$ generalize those for non-disjunctive logic programs with aggregates by \cite{denecker2002ultimate}. 
\begin{proposition}\label{prop:DMT:generalizes:DMT}
Let a non-disjunctive logic program ${\cal P}$ be given. Then $(x,y)$ is a stable model according to \cite{denecker2002ultimate} iff $(x,y)\in S({\cal IC}^{\sf DMT}_{\cal P})(x,y)$. 
\end{proposition}

We have shown how semantics for disjunctive aggregate logic programs can be obtained using the framework of non-deterministic AFT, solving the open question  \citep{alviano2023aggregate} of how operator-based semantics for aggregate programs can be generalized to disjunctive programs. 
This means AFT can be unleashed upon disjunctive aggregate programs, as demonstrated in this paper, as demonstrated in this section. 
Other semantics, such as the weakly supported semantics, the well-founded state  semantics \citep{nondetAFTarxiv} and semi-equilibrium semantics (Section \ref{sec:seq}, as illustrated in \ref{semi-equilibrium:to:aggregates}) are obtained without any additional effort and while preserving desirable properties shown algebraically for ndaos.
None of these semantics have, to the best of our knowledge, been investigated for dlp's with aggregates.
Other ndao's, left for future work, can likely be obtained straightforwardly on the basis of deterministic approximation operators for aggregate programs that we did not consider in this paper (e.g.\ the operator defined by \cite{vanbesien2021analyzing} to characterise the  semantics of \cite{marek2004set}  or the bounded ultimate operator introduced by \cite{pelov2004semantics}).

\section{Conclusion, in view of related work}
In this paper, we have made three contributions to the theory of non-deterministic AFT: (1) definition of the ultimate operator, (2) an algebraic generalization of the semi-equilibrium semantics and (3) an application of non-deterministic AFT to DLPs with aggregates in the body. To the best of our knowledge, there are only a few other semantics that allow for disjunctive rules with aggregates. Among the best-studied is the semantics by \cite{faber2004recursive} (so-called {\sf FLP}-semantics). As the semantics we propose generalize the operator-based semantics for aggregate programs without disjunction, the differences between the {\sf FLP}-semantics and the semantics proposed here essentially generalize from the non-disjunctive case (see e.g.\ \citep{alviano2023aggregate}).

Among the avenues for future work, an in-depth analysis of the computational complexity of the semantics proposed in this paper seems to be among the most pressing of questions. Other avenues of future work include the generalisation of the constructions in Section \ref{sec:aggregates} to other semantics \citep{vanbesien2021analyzing,alviano2023aggregate} and defining ndaos for rules with choice constructs in the head \citep{DBLP:journals/tplp/MarekNT08}, which can be seen as aggregates in the head.

\bibliographystyle{tlplike}
\bibliography{choice}

\appendix

\section{Proofs of Results in the Paper}

\begin{proof}[Proof of Proposition \ref{proposition:ultimate:is:most:precise}]
It is immediate to see that ${\cal O}^{\cal U}$ is exact and approximates $O$  (as $\bigcup_{x\leq z\leq x}O(z)=O(x)$). We now show it is $\preceq^A_i$-monotonic. For this, consider some $(x_1,y_1)\leq_i (x_2,y_2)$. Then $[x_1,y_1]\supseteq [x_2,y_2]$ and thus $\bigcup_{x_1\leq z\leq y_1} O(z)\supseteq \bigcup_{x_2\leq z\leq y_2}O(z) $ which implies ${\cal O}^{\cal U}_l(x_1,y_1) \preceq^S_L {\cal O}^{\cal U}_l(x_2,y_2)$.
The case for the upper bound is entirely identical.

We now show that any ndao ${\cal O}$ that approximates $O$ be given and for every $x,y\in{\cal L}$ s.t.\ $x\leq y$, ${\cal O}(x,y)\preceq^A_i{\cal O}^{\cal U}(x,y)$.
Indeed, consider an operator ${\cal O}$ that approximates $O$ and some $x,y\in {\cal L}$ s.t.\ $x\leq y$. Consider some $w\in {\cal O}^{\cal U}_l(x,y)$, i.e.\ $w\in O(z)$ for some $z\in [x,y]$. Since $x\leq z\leq y$, $(x,y)\leq_i (z,z)$ and thus ${\cal O}(x,y)\preceq^A_i {\cal O}(z,z)$ (with $\preceq^A_i$-monotonicity) which means ${\cal O}_l(x,y)\preceq^S_L {\cal O}_l(z,z)=O(z)$ (the latter equality since ${\cal O}$ approximates $O$) and thus, in particular, ${\cal O}_l(x,y)\preceq^S_L \{w\}$. 
Likewise, to show that $ {\cal O}^{\cal U}_l(x,y)\preceq^H_L {\cal O}_u(x,y)$, it suffices to observe that, since ${\cal O}(x,y)\preceq^A_i {\cal O}(z,z)$, $O(z)\preceq^H_L {\cal O}_u(x,y)$.
\end{proof}

\begin{proof}[Proof of Proposition \ref{prop:three-valued:are:subset:of:HT}]
suppose  $(x,y)$ is a model of ${\cal P}$ and consider some $\bigwedge\Theta_1\cup\Theta_2\rightarrow \bigvee\Delta\in {\cal P}$ (where $\Theta_1$ consists of atoms and $\Theta_2$ consists of negated atoms).
We first show that $(y,y)(\lnot \bigwedge\Theta_1\cup\Theta_2\lor \bigvee\Delta)={\sf T}$. To see this, suppose that $(y,y)( \bigwedge\Theta_1\cup\Theta_2)={\sf T}$. Then $(x,y)( \bigwedge\Theta_1\cup\Theta_2)\in \{{\sf T},{\sf U}\}$ and thus, since $(x,y)$ is a model of ${\cal P}$, $(x,y)(\bigvee\Delta)\in \{{\sf T},{\sf U}\}$ which implies $(y,y)(\bigvee\Delta)={\sf T}$.
We now show $(x,y)\not\models_{\sf HT} \bigwedge\Theta_1\cup\Theta_2$ or $(x,y)\models_{\sf HT}\bigvee\Delta$. Indeed, suppose that $(x,y)\not\models_{\sf HT} \bigwedge\Theta_1\cup\Theta_2$, i.e.\ $\Theta_1\subseteq x$ and for every $\lnot \alpha\in\Theta_2$, $\alpha\not\in y$. Thus, $(x,y)( \bigwedge\Theta_1\cup\Theta_2={\sf T}$ which implies, since $(x,y)$ is a model of ${\cal P}$, $(x,y)(\bigvee\Delta)={\sf T}$, i.e.\ $x\cap \Delta\neq\emptyset$. Thus, $(x,y)\models_{\sf HT} \bigvee\Delta$. 

To see that the inclusion can be proper, consider ${\cal P}=\{b\leftarrow \lnot c\}$. Then $(\emptyset,\{c\}) \in {\sf HT}({\cal P})$ as $\{c\}$ does not classically entail $\lnot c$. However, $(\emptyset,\{c\})(\lnot c)={\sf U}$ whereas $(\emptyset,\{c\})(b)={\sf F}$ and thus $(\emptyset,\{c\})$ is not a model of ${\cal P}$. 
\end{proof}

\begin{proof}[Proof of Proposition \ref{prop:ht:models:for:lp:are:faithful}]
Proposition \ref{prop:ht:models:for:lp:are:faithful} follows from the following Lemma:
\begin{lemma}
For consistent $(x,y)$, ${\cal IC}_{\cal P}^l(x,y)\preceq^S_L(x,y)$ iff $(x,y)\not\modelsht \bigwedge_{i=1}^n\alpha_i\land \bigwedge_{j=1}^m \lnot \beta_j$ or $(x,y)\modelsht \bigvee \Delta$ for any $\bigvee\Delta\leftarrow \bigwedge_{i=1}^n\alpha_i\land \bigwedge_{j=1}^m \lnot \beta_j\in {\cal P}$. 
\end{lemma}
\begin{proof}
For the $\Leftarrow$-direction, suppose ${\cal IC}_{\cal P}^l(x,y)\preceq^S_L(x,y)$ and consider some $\bigvee\Delta\leftarrow  \bigwedge_{i=1}^n\alpha_i\land \bigwedge_{j=1}^m \lnot \beta_j\in {\cal P}$. Suppose that $(x,y)\modelsht \bigwedge_{i=1}^n\alpha_i\land \bigwedge_{j=1}^m \lnot \beta_j$, i.e.\ $\alpha_i\in x$ for every $1\leq i\leq n$ and $\beta\not\in y$ for every $1\leq j\leq m$. Then $(x,y)(\bigwedge_{i=1}^n\alpha_i\land \bigwedge_{j=1}^m \lnot \beta_j)={\sf T}$ (since $x\subseteq y$, $(x,y)(\bigwedge_{i=1}^n\alpha_i\land \bigwedge_{j=1}^m \lnot \beta_j)={\sf C}$ is excluded as $(x,y)$ is consistent) and thus $\Delta\in {\cal HD}_{\cal P}^l(x,y)$, which means, since ${\cal IC}_{\cal P}^l(x,y)\preceq^S_L(x,y)$, there is some $(w,z)\in {\cal IC}_{\cal P}^l(x,y)$ s.t.\ $\Delta\cap w\neq \emptyset$ and $w\subseteq x$, i.e.\ $(x,y)\modelsht \bigvee\Delta$.  The $\Rightarrow$-direction is analogous. 
\end{proof}\end{proof}

\begin{proof}[Proof of Proposition \ref{prop:ht-models:and:stable:fixpoints}]
Ad (1). We will make use of the following Lemma:
\begin{lemma}[{\cite{nondetAFTarxiv}}]
\label{lemma:minimalpre-fix:is:minima:fix}
Let $O:{\cal L}\rightarrow \wp({\cal L})$ be a $\preceq^S_L$-monotonic non-deterministic operator. 
Then if $w$ is a $\leq$-minimal pre-fixpoint of $O$, it is a $\leq$-minimal fixpoint of $O$.
\end{lemma}
We now proceed to the proof of the proposition. We first show that $ \min_{\leq_t}({\sf HT}({\cal O}))\subseteq S({\cal O})(x,x)$. Suppose $(x,x)\in \min_{\leq_t}({\sf HT}({\cal O}))$.

We first show that $x\in C({\cal O}_l)(x)$. As ${\cal O}^l(x,x)\preceq^S_t x$, i.e.\ $x$ is a pre-fixpoint of ${\cal O}^l(.,x)$. Suppose towards a contradiction that there is some $z<x$ s.t.\ ${\cal O}^l(z,x)\preceq^S_t z$. Then $(z,x)\in {\sf HT}({\cal O})$, contradicting $(x,x)\in \min_{\leq_t}({\sf HT}({\cal O}))$. Thus, $x$ is a $\leq$-minimal prefixpoint, and thus, with Lemma \ref{lemma:minimalpre-fix:is:minima:fix}, $x$ is a $\leq$-minimal fixpoint of ${\cal O}^l(.,x)$, which implies $x\in C({\cal O}_l)(x)$.

We now show that  $x\in C({\cal O}_u)(x)$. As ${\cal O}_l(x,x)={\cal O}_u(x,x)$ (since an ndao ${\cal O}$ is exact), $x\in {\cal O}_u(x,x)$. Suppose now towards a contradiction there is some $y<x$ s.t.\ $y\in{\cal O}_u(x,y)$. As ${\cal O}$ is upwards coherent, ${\cal O}_l(x,y)\preceq^S_L {\cal O}_u(x,y)$ which implies that there is some $y'\in {\cal O}_u(x,y)$, i.e.\ ${\cal O}_u(x,y)\preceq^S_L \{y'\}\preceq^S_L \{y\}$, i.e.\ $y$, which contradicts $(x,x)\in \min_{\leq_t}({\sf HT}({\cal O}))$.

We now show that $ \min_{\leq_t}({\sf HT}({\cal O}))\supseteq S({\cal O})(x,x)$. Suppose $(x,x)\in S({\cal O})(x,x)$ and suppose towards a contradiction there is some $(w,z)\in {\sf HT}({\cal O})$ s.t.\ $w\leq x$ and $z\leq x$, and either $z\neq x$ or $w\neq x$. Suppose first the latter. Then ${\cal O}_l(w,z)\preceq^S_L w$ and thus $w$ is a pre-fixpoint of ${\cal O}_l(w,z)$. Since $z\leq x$, ${\cal O}(w,x)\preceq^S_L {\cal O}(w,z)$ and thus  ${\cal O}_l(w,x)\preceq^S_L w$, i.e.\ $w$ is a pre-fixpoint of ${\cal O}_l(.,x)$. But this contradicts $x$ being a minimal fixpoint of ${\cal O}_l(.,x)$ (with Lemma \ref{lemma:minimalpre-fix:is:minima:fix}). Suppose now $z\leq x$. In view of the previous case, we can assume $w=x$. But then $(w,z)\in {\sf HT}({\cal O})$ contradicts $z\leq x=w$, contradiction.

Ad (2). Since $(x,x)\in {\cal O}(x,x)$, $x\in {\cal O}_l(x)$ and thus ${\cal O}_l(x,x)\preceq^S_L x$. Since ${\cal O}(x,x)=O(x)\times O(x)$, also $O(x)\preceq^S_L x$. 

\end{proof}

\begin{proof}[Proof of Proposition \ref{prop:algebraic:seq:is:nice}]
We first recall some background from \cite{nondetAFTarxiv}.  
\begin{definition}%
\label{def:downwards:closed}
A non-deterministic operator $O:{\cal L}\rightarrow \wp({\cal L})$ is \emph{downward closed} if for every sequence $X=\{x_\epsilon\}_{\epsilon<\alpha}$ of elements in ${\cal L}$ such that:
\begin{enumerate}
\item for every $\epsilon<\alpha$, $O(x_\epsilon)\preceq^S_L \{x_\epsilon\}$, and
\item for every $\epsilon <\epsilon'<\alpha$, $x_{\epsilon'}< x_\epsilon$,
\end{enumerate}
it holds that $O(glb(X))\preceq^S_L glb(X)$.
\end{definition}
It is clear that an operator over a finite lattice is always downwards closed. 
We now recall the following result from \cite{nondetAFTarxiv}:
\begin{proposition}
\label{proposition:downward:closed:then:fp}
Let ${\cal L}=\langle L,\leq\rangle$ be a lattice and let a $O:{\cal L}\rightarrow \wp({\cal L})$ be a downward closed, $\preceq^S_L$-monotonic non-deterministic operator.
Then $O$ admits a $\leq$-minimal fixpoint.
\end{proposition}
This means that ${\cal O}_l(.,\top)$ admits a $\leq$-minimal fixpoint for any operator over a finite lattice (as it is $\preceq^S_L$-monotonic in view of \cite[Lemma 3]{nondetAFTarxiv}. Let $x$ be such a fixpoint. Then $x\leq \top$, $O(\top)\preceq^S_L \top$ and ${\cal O}(x,y)\preceq^S_L x$, i.e.\ $(x,\top)\in {\sf HT}({\cal O})$ and thus 
$\mathcal{SEQ}({\cal O})\neq\emptyset$. 

We now show that $(x,x)\in {\sf mc}(\min_{\leq_t}({\sf HT}({\cal O})$ implies $(x,x)\in S({\cal O})(x,x)$. Indeed, as $x\oslash x=\bot$, $(x_1,y_1)\in  {\sf mc}(\min_{\leq_t}({\sf HT}({\cal O})$ iff $x_1\oslash y_1=\bot$, which implies $x_1=y_1$.  The rest follos from Proposition \ref{prop:ht-models:and:stable:fixpoints}. 
\end{proof}

\begin{proof}[ Proof of Proposition \ref{prop:aggregate:operators:are:ndaos}]
We first show the first item. The case for $\xi={\cal U}$ follows from Proposition \ref{proposition:ultimate:is:most:precise}. The case for $\xi={\sf DMT}$ is the same as the case for disjunctive logic programs without aggregates (Proposition 3 of \cite{nondetAFTarxiv}).
For $\xi={\sf GZ}$, we first show $\preceq^A_i$-monotonicity. For this, consider some $(x_1,y_1)\leq_i (x_2,y_2)$. If $x_1\neq y_1$ we are done. Suppose thus that $x_1=y_1$. But then $x_1\subseteq x_2$ and $y_2\subseteq y_1$ implies $x_1=x_2=y_1=y_2$ which means ${\cal IC}_{\cal P}^{\sf GZ}(x_1,y_1)={\cal IC}_{\cal P}^{\sf GZ}(x_2,y_2)$. Exactness is clear, and so is the fact that ${\cal IC}_{\cal P}^{\xi}(x,y)$ approximates $\IC_{\cal P}$ (as ${\cal HD}^{{\xi}}_{\cal P}(x,x)=\HR_{\cal P}(x,x)$.

We now show the second item. From Proposition \ref{proposition:ultimate:is:most:precise}, it follows that ${\cal IC}^{\sf DMT}_{\cal P}(x,y)\preceq^A_i {\cal IC}^{\cal U}_{\cal P}(x,y)$. We now show that ${\cal IC}^{\sf GZ}_{\cal P}(x,y)\preceq^A_i {\cal IC}^{\sf DMT}_{\cal P}(x,y)$. We first consider the case $x\neq y$. Then  ${\cal IC}^{\sf GZ}_{\cal P}(x,y)=\{\emptyset\}\times \{{\cal A}_{\cal P}\}$ and we are done. If $x=y$, ${\cal IC}^{\sf GZ}_{\cal P}(x,y)= {\cal IC}^{\sf DMT}_{\cal P}(x,y)=\IC_{\cal P}(x,x)\times \IC_{\cal P}(x,x)$ as both operators approximate $\IC_{\cal P}$.
\end{proof}

\begin{proof}[Proof of Proposition \ref{prop:GZ:generalizes:GZ}]
The first item is immediate from the fact that ${\cal IC}^{\sf GZ}_{\cal P}(x,y)=(\emptyset,{\cal A}_{\cal P})$ if $x\neq y$.

We first show the following useful Lemma:
\begin{lemma}\label{lemma:post-fixpoint:iff:model:of:GZ:reduct}
Let some $y\subseteq x\subseteq {\cal A}_{\cal P}$ be given. Then
${\cal IC}^{{\sf GZ},l}_{\cal P}(y,x)\preceq^S_L y$ iff $y$ is a model of  $\frac{{\cal P}^x_{\sf GZ}}{x}$.
\end{lemma}
\begin{proof}
For the $\Leftarrow$-direction, suppose ${\cal IC}^{{\sf GZ},l}_{\cal P}(y,x)\preceq^S_L y$. Consider some $\bigvee\Delta\leftarrow \bigwedge_{i=1}^n\alpha_i\land \bigwedge_{j=1}^m \lnot \beta_j\in {\cal P}$ s.t.\ for every aggregate atom $\alpha_i$, $x(\alpha_i)={\sf T}$. Suppose furthermore $\beta_j\not\in x$. Thus, $(y,x)(\bigwedge_{j=1}^m \lnot \beta_j)={\sf T}$. 
 Let $f(\alpha_i)=\alpha_i$ if $\alpha_i$ is not an aggregate atom, and $f(\alpha_i)=\bigcup\{\psi\subseteq Conj\mid x(\psi)={\sf T}\}$ if $\alpha_i=f(\langle Vars: Conj\rangle)\ast w$. Suppose that $y(f(\alpha_i))={\sf T}$ for every $1\leq i\leq n$. This means that ${\sf At}(\alpha_i)\cap y\supseteq{\sf At}(\alpha_i)\cap x$ and thus, since $y\subseteq x$, ${\sf At}(\alpha_i)\cap y={\sf At}(\alpha_i)\cap x$. Thus (in view of $x(\alpha_i)={\sf T}$, $(x,y)\models_{\sf GZ} \alpha_i$ for every $1\leq i\leq n$  and thus $\Delta\in {\cal HD}^{{\sf GZ},l}_{\cal P}$, which means $y\cap \Delta\neq\emptyset$ (as ${\cal IC}^{{\sf GZ},l}_{\cal P}(y,x)\preceq^S_L y$).
 
 For the $\Rightarrow$-direction, suppose $y$ is a model of $\frac{{\cal P}^x_{\sf GZ}}{x}$. We show that for every $\bigvee\Delta\leftarrow \bigwedge_{i=1}^n\alpha_i\land \bigwedge_{j=1}^m \lnot \beta_j\in {\cal P}$ s.t.\ $(y,x)\models_{\sf GZ} \bigwedge_{i=1}^n\alpha_i\land \bigwedge_{j=1}^m \lnot \beta_j$, $\Delta\cap y\neq\emptyset$. Indeed, suppose that $(y,x)\models_{\sf GZ} \bigwedge_{i=1}^n\alpha_i\land \bigwedge_{j=1}^m \lnot \beta_j$. This implies, (1) $\beta_j\not\in x$ for any $1\leq j\leq m$, and (2) $(y,x)\models_{\sf GZ} \alpha_i$ for every $1\leq i \leq n$. (2) means, in particular, that $x(\alpha_i)={\sf T}$ and ${\sf At}(\alpha_i)\cap x={\sf At}(\alpha_i)\cap y$ (for $i=1,\ldots,n$). Thus, $\bigvee\Delta\leftarrow \bigwedge_{i=1}^nf(\alpha_i)\in {\cal P}^x_{\sf GZ}$ (where $f(\alpha_i)$ is defined as in the proof of the $\Leftarrow$-direction). But then, since $y$ is a model of $\frac{{\cal P}^x_{\sf GZ}}{x}$, $y\cap \Delta\neq\emptyset$.
\end{proof}

We will furthermore use the following Lemma:
\begin{lemma}[{\cite{nondetAFTarxiv}}]
\label{minimalfp:then:minimal:prefp}
Let $O:{\cal L}\rightarrow \wp({\cal L})$ be a $\preceq^S_L$-monotonic non-deterministic operator. Then if $w$ is a $\leq$-minimal fixpoint of $O$, it is a $\leq$-minimal pre-fixpoint of $O$.
\end{lemma}

For the $\Rightarrow$-direction, suppose $(x,x)\in  S({\cal IC}^{\sf GZ}_{\cal P})(x,x)$, i.e.\ $x$ is a $\subseteq$-minimal fixpoint of ${\cal IC}^{{\sf GZ},l}_{\cal P}(.,x)$. We show that $x$ is an answer set of ${\cal P}^x_{\sf GZ}$. With Lemma \ref{lemma:post-fixpoint:iff:model:of:GZ:reduct}, $x$ is a model of $\frac{{\cal P}^x_{\sf GZ}}{x}$. Suppose now there is some $y\subset x$ s.t.\ $y$ is a model of $\frac{{\cal P}^x_{\sf GZ}}{x}$. Then, again with Lemma \ref{lemma:post-fixpoint:iff:model:of:GZ:reduct}, ${\cal IC}^{{\sf GZ},l}_{\cal P}(y,x)\preceq^S_L y$, i.e.\ $y$ is a pre-fixpoint of ${\cal IC}^{{\sf GZ},l}_{\cal P}(.,x)$. But since $x$ is a minimal fixpoint of ${\cal IC}^{{\sf GZ},l}_{\cal P}(.,x)$, with Lemma \ref{minimalfp:then:minimal:prefp}, $x$ is a $\subseteq$-minimal pre-fixpoint, contradiction.

For the $\Leftarrow$-direction, suppose that $x$ is a {\sf GZ}-answer set of ${\cal P}$. Then with Lemma \ref{lemma:post-fixpoint:iff:model:of:GZ:reduct}, ${\cal IC}^{{\sf GZ},l}_{\cal P}(x,x)\preceq^S_L x$. Suppose that there is some $y\subset x$ s.t.\ ${\cal IC}^{{\sf GZ},l}_{\cal P}(y,x)\preceq^S_L y$. Then, again with Lemma \ref{lemma:post-fixpoint:iff:model:of:GZ:reduct}, $y$ is a model of $\frac{{\cal P}^x_{\sf GZ}}{x}$, contradiction. Thus, $x$ is a $\subseteq$-minimal pre-fixpoint of ${\cal IC}^{{\sf GZ},l}_{\cal P}(.,x)$, which means, with Lemma \ref{lemma:minimalpre-fix:is:minima:fix}, that $x$ is a minimal fixpoint of ${\cal IC}^{{\sf GZ},l}_{\cal P}(.,x)$, and thus $(x,x)\in S({\cal IC}^{{\sf GZ},l}_{\cal P})(x,x)$.
\end{proof}

\begin{proof}[Proof of Proposition \ref{prop:DMT:generalizes:DMT}]
This is immediate from the definition of ${\cal IC}^{\sf DMT}_{\cal P}$. 
\end{proof}

\section{Addition Details for Section \ref{sec:seq}}
\subsection{On the Absence of the Upper Bound in HT-pairs and Upwards-Coherence}
The reader might be suprised by the fact that only the lower bound ${\cal O}_l$ occurs in the definition of a ${\sf HT}$-pair. 
The reason is that ${\sf HT}$-pairs have only been used in results relative to total stable interpretations. For such interpretations, the exact form of the upper bound-operator is inconsequential as long as it extends the lower bound-operator. We call such ndaos upwards coherent (as already defined in the main paper, and recalled here):
\begin{definition}
An ndao ${\cal O}$ over ${\cal L}$ is upwards coherent if for every $x,y\in{\cal L}$, ${\cal O}_l(x,y)\preceq^S_L {\cal O}_u(x,y)$. 
\end{definition}
Thus, an ndao ${\cal O}$ is upwards coherent if for every upper bound $y_1\in{\cal O}_u(x,y)$, there is a lower bound $x_1\in {\cal O}_l(x,y)$ s.t.\ $x_1\leq y_1$. 
This property was shown to hold for symmetric ndaos \cite[Proposition 5]{nondetAFTarxiv}. However, there also exist non-symmetric upwards coherent operators. The operator $\IC^{\sf DMT^d}_{\cal P}$ is a case in point.

\begin{proposition}\label{prop:DMT:upwards:cohrerent}
$\IC^{\sf DMT}_{\cal P}$ is upwards coherent.
\end{proposition}\begin{proof}
This is immediate from the observation that if $\phi$ is true in every interpretation $x\subseteq z \subseteq y$, then $\phi$ is true in at least one interpretation $x\subseteq z \subseteq y$
\end{proof}

\begin{proposition}\label{prop:for:total:stable:only:lower:bound:matters}
Let an upwards coherent ndao ${\cal O}$ be given. Then $(x,x)\in S({\cal O})(x,x)$ iff $x\in C({\cal O}_l)(.,x)$. 
\end{proposition}
\begin{proof}
The $\Rightarrow$-direction is immediate. Consider now an upwards coherent ndao ${\cal O}$ and suppose $x\in C({\cal O}_l)(x)$. With Lemma \ref{minimalfp:then:minimal:prefp}, this means that $x$ is a minimal pre-fixpoint of ${\cal O}_l(.,x)$.
We have to show that $x\in C({\cal O}_u)(x)$. First, notice that since ${\cal O}$ is an ndao, ${\cal O}_l(x,x)=O(x)\times O(x)={\cal O}_u(x,x)$, which implies $x\in {\cal O}_u(x,x)$. Suppose now towards a contradiction that there is some $y<x$ s.t.\ $y\in {\cal O}_u(x,y)$. 
As ${\cal O}_l(x,.)$ is $\preceq^S_L$-anti-monotonic (\cite[Lemma 3]{nondetAFTarxiv}), ${\cal O}_l(x,x)\preceq^S_L{\cal O}_l(x,y)$, which implies there is some $y'\in {\cal O}_l(x,x)$ s.t.\ $y'\leq y$ (and thus $y'<x$). As ${\cal O}_l(.,x)$ is $\preceq^S_L$-monotonic (\cite[Lemma 3]{nondetAFTarxiv}), ${\cal O}_l(y',x)\preceq^S_L{\cal O}_l(x,x)$. As $y'\in {\cal O}_l(x,x)$, we see that ${\cal O}_l(y',x)\preceq^S_L y'$, which means that $y'$ is a pre-fixpoint of ${\cal O}_l(.,x)$. But this contradicts $x$ being a minimal pre-fixpoint.
\end{proof}

\subsection{Semi-Equilibrium Semantics for Lattices without Difference Operators}
In this section, we look closer into the characterisation of the semi-equilibrium semantics for operators over lattices without a difference operator.
We first look closer into the relation between the gap of an interpretation and the  information ordering $\leq_i$ is therefore not unexpected:
\begin{proposition}\label{prop:from:gap:to:inf:and:back}
Given two consistent interpretations $(x_1,y_1)$ and $(x_2,y_2)$, if $(x_2,y_2)\leq_i (x_1,y_1)$ then $gap(x_1,y_2)\subseteq gap(x_2,y_2)$, but not vice versa.
\end{proposition}
\begin{proof}
Suppose $(x_2,y_2)\leq_i (x_1,y_1)$. Then $x_2\subseteq x_1\subseteq y_1\subseteq x_2$ and thus $y_1\setminus x_1\subseteq y_2\setminus x_2$. 
which implies $gap(x_1,y_1)\subseteq gap(x_2,y_2)$.

The other direction might not hold, as  two interpretations (e.g.\ $(\{r\},\{r,q\})$ and $(\{p\},\{p\})$) might have a comparable gap (e.g.\ $gap((\{r\},\{r,q\}))=\{q\}\supset gap((\{p\},\{p\}))$) but be incomparable w.r.t.\ $\leq_i$. 
\end{proof}

We can now show that semi-equilibrium models can be approximated algebraically:

\begin{corollary}
Let a normal disjunctive logic program ${\cal P}$ be given. Then \[\max_{\leq_i}\left(\min_{\leq_t}({\sf HT}({\cal IC}_{\cal P})\right)\supseteq {\cal SEQ}({\cal P}).\]
\end{corollary}

The properties deemed desirable by \cite{amendola2016semi} hold for this algebraic generalisation in the following sense (for finite lattices):
\begin{proposition}\label{prop:algebraic:seq:is:nice:ish}
Let an ndao ${\cal O}$ over a finite lattice be given. Then  \[\max_{\leq_i}\left(\min_{\leq_t}( {\sf HT}({\cal O})  )\right)\neq\emptyset.\]
Furthermore, for any stable fixpoint $(x,x)$ of ${\cal O}$, $(x,x)\in \max_{\leq_i}\left(\min_{\leq_t}( {\sf HT}({\cal O})  )\right)$, and for any $(x,x)\in \max_{\leq_i}\left(\min_{\leq_t}( {\sf HT}({\cal O})  )\right)$, $(x,x)$ is a stable fixpoint.
\end{proposition}
\begin{proof}
The proof is entirely analogous to the proof of Proposition \ref{prop:algebraic:seq:is:nice}.
\end{proof}

\section{Semi-Equilibrium Semantics for Disjunctive Aggregate Programs}\label{semi-equilibrium:to:aggregates}
In this appendix, we show some basic results on semi-equilibrium semantics induced by the operators from Section \ref{sec:aggregates}, and illustrate these semantics with an example.

We first show the following lemma which establishes that the three operators $\ICc^{\sf GZ}_{\cal P}$, $\ICc^{\sf DMT}_{\cal P}$ and $\ICc^{\cal U}_{\cal P}$ defined in Section are upwards coherent:
\begin{lemma}\label{lemma:aggregate:ndaos:are:upwards:coherent}
For any $\dagger\in \{{\sf GZ},{\sf DMT}, {\cal U}$ and any disjunctive aggregate program ${\cal P}$, $\ICc^{\dagger}_{\cal P}$ is upwards coherent, that is, $\ICc^{\dagger}_{l,{\cal P}}(x,y)\preceq^S \ICc^{\dagger}_{u,{\cal P}}(x,y)$ for any $x\subseteq {\cal A}_{\cal P}$.
\end{lemma}
\begin{proof}
To see that $\ICc^{\sf GZ}_{\cal P}$ is upwards coherent, it suffices to observe that either $\ICc^{\sf GZ}_{l,{\cal P}}(x,y)= \ICc^{\sf GZ}_{u,{\cal P}}(x,y)$ (when $x=y$) or $=\{\emptyset\}=\ICc^{\dagger}_{l,{\cal P}}(x,y)\preceq^S \ICc^{\dagger}_{u,{\cal P}}(x,y)=\{{\cal A}_{\cal P}\}$ (when $x\neq y$).

We now show that $\ICc^{\sf DMT}_{\cal P}$ is upwards coherent. We first show that $\HDc^{\sf DMT}_{l,{\cal P}}(x,y)\subseteq \HDc^{\sf DMT}_{u,{\cal P}}(x,y)$ for any $x\subseteq y\subseteq{\cal A}_{\cal P}$. 
Indeed, if $\Delta\in \HDc^{\sf DMT}_{l,{\cal P}}(x,y)$, then for every $x\subseteq z \subseteq y$, there is some $\bigvee\Delta\leftarrow \phi\in{\cal P}$ s.t.\ $z(\phi)={\sf T}$. 
Thus, there is some $x\subseteq z \subseteq y$ and some $\bigvee\Delta\leftarrow \phi\in{\cal P}$ s.t.\ $z(\phi)={\sf T}$ and thus $\Delta\in \HDc^{\sf DMT}_{u,{\cal P}}(x,y)$.
 We now show $\ICc^{\sf DMT}_{l,{\cal P}}(x,y)\preceq^S \ICc^{\sf DMT}_{u,{\cal P}}(x,y)$.
  Indeed, consider some $y_1\in  \ICc^{\sf DMT}_{u,{\cal P}}(x,y)$. 
  It is easy to see that $y_1\cap \bigcup \HDc^{\sf DMT}_{l,{\cal P}}(x,y)\in ^{\sf DMT}_{l,{\cal P}}(x,y)$ as $y_1\subseteq  \HDc^{\sf DMT}_{u,{\cal P}}(x,y)$  and we have shown above that $\HDc^{\sf DMT}_{l,{\cal P}}(x,y)\subseteq \HDc^{\sf DMT}_{u,{\cal P}}(x,y)$.

Finally, $\ICc^{\cal U}$ is clearly upwards coherent as for any $x\subseteq y\subseteq {\cal A}_{\cal P}$, $\ICc^{\cal U}_{l,{\cal P}}(x,y)= \ICc^{\cal U}_{u,{\cal P}}(x,y)$.
\end{proof}

We can now show that the semi-equilibrium semantics based on the ndao's from Section \ref{sec:aggregates} satisfy all properties deemed desirable by \cite{amendola2016semi}:
\begin{proposition}
For any $\dagger\in \{{\sf GZ},{\sf DMT}, {\cal U}\}$ and any disjunctively normal logic program ${\cal P}$,  ${\cal SEQ}(\ICc^{\dagger}_{\cal P})\neq\emptyset$. 
Furthermore, if there is some $(x,x)\in S(\ICc^{\dagger}_{\cal P})  ))$ then ${\cal SEQ}(\ICc^{\dagger}_{\cal P})=\{(x,x)\in {\cal L}^2\mid (x,x)\in S(\ICc^{\dagger}_{\cal P})(x,x)\}$.
\end{proposition}
\begin{proof}
Follows immediately from Lemma \ref{lemma:aggregate:ndaos:are:upwards:coherent} and Proposition \ref{prop:algebraic:seq:is:nice}.
\end{proof}
In particular, this means that whenever ${\cal P}$ admits a ${\sf GZ}$-answer set, ${\cal SEQ}(\ICc^{\sf GZ}_{\cal P})$ coincides with the ${\sf GZ}$-answer sets (and similarly for ${\sf DMT}$ and ${\cal U}$).

We now illustrate these semantics with an example:
\begin{example}[Example \ref{ex:running:aggregates} continued]
Let ${\cal P}=\{p\lor q\leftarrow   \#{\tt Sum}[1:p, 1:q]<0; s\leftarrow  \#{\tt Sum}[1:p, 1:q]<0\}$ be given.

We first determine the heads and immediate consequences of all relevant interpretations:
\begin{adjustbox}{width=0.95\textwidth,center}
\begin{oldtabular}{l|lllllllll} \hline\hline
$x$ & $\emptyset$ & $\{s\}$ & $\{q\}$ & $\{p\}$ & $\{p,q\}$ & $\{p,s\}$ & $\{q,s\}$ & $\{s,q,p\}$& \\ \hline
$\HR_{\cal P}(x)$ & $\{\{p,q\},\{s\}\}$ & $ \{\{p,q\},\{s\}\}$ & $\emptyset$ & $\emptyset$ & $\emptyset$ & $\emptyset$ & $\emptyset$& $\emptyset$& \\ \hline
$\IC_{\cal P}(x)$ & $\Delta_1$ & $\Delta_1$ &$\{\emptyset\}$ &  $\{\emptyset\}$ & $\{\emptyset\}$ &  $\{\emptyset\}$ &  $\{\emptyset\}$ &  $\{\emptyset\}$& \\ \hline\hline
\end{oldtabular}
\end{adjustbox}
where $\Delta_1=\{\{p,s\},\{q,s\},\{p,q,s\}\}$.

We now determine the ${\sf HT}$-interpretations of the respective ndao's:
\begin{itemize}
\item ${\sf HT}(\ICc^{\sf GZ}_{\cal P})=\{(x,y)\mid x\subseteq y\subseteq {\cal A}_{\cal P}, y\not\in\{ \emptyset, \{s\}\}\mbox{ and }(x\neq y \mbox{ or }x \neq \emptyset, \{s\}) \}$. 
\item ${\sf HT}(\ICc^{\sf GZ}_{\cal P})=\{(x,y)\mid x\subseteq y\subseteq {\cal A}_{\cal P},y\not\in\{ \emptyset, \{s\}\} \}$.
\item ${\sf HT}(\ICc^{\sf GZ}_{\cal P})=\{(x,y)\mid x\subseteq y\subseteq {\cal A}_{\cal P}, y\not\in\{ \emptyset, \{s\}\}\}$.
\end{itemize}
Notice that, among others, $(\emptyset,\{p\}), (\emptyset,\{q\})\in {\sf HT}(\IC^{\dagger}_{\cal P})$ while $(\emptyset,\emptyset),(\emptyset,\{s\})\not\in  {\sf HT}(\ICc^{\dagger}_{\cal P})$ for any $\dagger\in \{{\sf GZ},{\sf DMT}, {\cal U}\}$. We thus see that $\min_{\leq_t}({\sf HT}(\ICc^{\dagger}_{\cal P})=\{(\emptyset,\{p\}),(\emptyset,\{q\})\}$ and that ${\cal SEQ}(\ICc^{\dagger}_{\cal P})=\{(\emptyset,\{p\}),(\emptyset,\{q\})\}$ (again for any $\dagger\in \{{\sf GZ},{\sf DMT}, {\cal U}\}$).
\end{example}

\section{Summary of Notational Conventions and Definitions}

\def\arraystretch{1,3}\tabcolsep=10pt
\begin{table}[htb]
\begin{tabular}{lll}
Elements & Notations & Example \\ \hline \hline
Elements of ${\cal L}$ & $x,y,\ldots$ & $x,y$\\
Sets of elements of ${\cal L}$ & $X,Y,\ldots$ & $\{x_1,y_1,x_2,y_2\}$ 
\end{tabular}
\caption{List of the notations of different types of sets used in this paper} 
\label{tab:set-notations}
\end{table}

\begin{table}[htb]
\begin{tabular}{lll}
Preorder & Type & Definition \\ \hline \hline
\multicolumn{3}{c}{Element Orders} \\ \hline
$\leq$ & ${\cal L}$ & primitive \\
$\leq_i$ & ${\cal L}^2\times {\cal L}^2$ & $(x_1,y_1) \leq_i (x_2,y_2)$ iff $x_1 \leq x_2$ and $y_1 \geq y_2$ \\
$\leq_t$ & ${\cal L}^2\times {\cal L}^2$ & $(x_1,y_1) \leq_t (x_2,y_2)$ iff $x_1 \leq x_2$ and $y_1 \leq y_2$ \\ \hline
\multicolumn{3}{c}{Set-based Orders} \\ \hline
$\preceq^S_L$ & $\wp({\cal L})\times \wp({\cal L})$ & $X \preceq^S_L Y$ iff for every $y\in Y$ there is an 
$x\in X$ s.t.\ $x\leq y$\\
$\preceq^H_L$ & $\wp({\cal L})\times \wp({\cal L})$ &  $X \preceq^H_L Y$ iff for every $x\in X$ there is an 
$y\in Y$ s.t.\ $x\leq y$ \\
$\preceq^A_i$ & $\wp({\cal L})^2\times\wp({\cal L})^2$ & $(X_1,Y_1)\preceq_i^A (X_2,Y_2)$ iff $X_1\preceq^S_L X_2$ and $Y_2\preceq^H_L Y_1$\\
\end{tabular}
\caption{List of the preorders used in this paper.}
\label{tab:orders}
\end{table}

\begin{table}[htb]
\begin{tabular}{lclll}
Operator & Notation& Type \\ \hline \hline
Non-deterministic operator & $O$ & ${\cal L}\mapsto \wp({\cal L})$  \\
Non-deterministic approximation operator & ${\cal O}$ & ${\cal L}^2\mapsto \wp({\cal L})\times\wp({\cal L})$ \\
Stable operator & $S({\cal O})$ & ${\cal L}^2\mapsto \wp({\cal L})\times\wp({\cal L})$ \\
Ultimate operator & ${\cal O}^{\cal U}$ & ${\cal L}^2\mapsto \wp({\cal L})\times\wp({\cal L})$
\end{tabular}
\caption{List of the operators used in this paper.}
\label{tab:operators}
\end{table}

\begin{table}[htb]
\begin{tabular}{ll}
\multicolumn{2}{c}{Immediate Consequence Operator for dlp's} \\ \hline \hline
$\HR_{\cal P}(x)$ & $=\{\Delta\mid \bigvee\!\Delta\leftarrow \psi \in{\cal P} \text{ and } (x,x)(\psi) =  {\sf T}\}$,\\
 $\IC_{\cal P}(x)$ &$=\{y\subseteq \:\bigcup\!\HR_{\cal P}(x)  \mid \forall \Delta \in \HR_{\cal P}(x), \ y \cap \Delta \neq \emptyset \}$.\\ \hline
 \multicolumn{2}{c}{Ndao ${\cal IC}_{\cal P}$ for disjunctive logic programs} \\ \hline \hline
$\HRc^l_{\cal P}(x,y)$ & $= \{ \Delta \mid \bigvee\!\Delta \leftarrow \phi\in {\cal P}, (x,y)(\phi)\geq_t {\sf C}\}$, \\
$\HRc^u_{\cal P}(x,y)$ & $= \{ \Delta \mid \bigvee\!\Delta \leftarrow \phi\in {\cal P}, (x,y)(\phi)\geq_t {\sf U}\}$, \\
${\cal IC}^\dagger_{\cal P}(x,y)$ &$=\{x_1\subseteq \bigcup\HRc^\dagger_{\cal P}(x,y) \mid \forall \Delta\in \HRc^\dagger_{\cal P}(x,y), \ x_1 \cap \Delta \neq \emptyset \}$ (for $\dagger\in \{l,u\}$),\\
${\cal IC}_{\cal P}(x,y)$ & $=({\cal IC}^l_{\cal P}(x,y), {\cal IC}^u_{\cal P}(x,y))$. \\ \hline
  \multicolumn{2}{c}{Ndao ${\cal IC}^{{\sf DMT}}_{\cal P}$ based on the deterministic ultimate approximation} \\ \hline \hline
 ${\cal HD}^{{\sf DMT},l}_{\cal P}(x,y)$ & $=\bigcap_{x\subseteq z\subseteq y}\HR_{\cal P}(z)$,\\
${\cal HD}^{{\sf DMT},u}_{\cal P}(x,y)$ & $=\bigcup_{x\subseteq z\subseteq y} \HR_{\cal P}(z)\}$,\\
${\cal IC}_{\cal P}^{{\sf DMT},\dagger}(x,y)$ & $=\{z\subseteq \bigcup{\cal HD}^{{\sf DMT},\dagger}_{\cal P}(x,y) \mid \forall \Delta\in{\cal HD}^{{\sf DMT},\dagger}_{\cal P}(x,y)\neq\emptyset: z\cap \Delta\neq \emptyset\}$.\\ \hline
  \multicolumn{2}{c}{ Ultimate ndao ${\cal IC}^{{\cal U}}_{\cal P}$} \\ \hline \hline
${\cal IC}_{\cal P}^{{\cal U}}(x,y)$ & $=\bigcup_{x\subseteq z\subseteq y}\IC_{\cal P}(z)\times \bigcup_{x\subseteq z\subseteq y}\IC_{\cal P}(z)$   \\ \hline
  \multicolumn{2}{c}{Trivial ndao ${\cal IC}_{\cal P}^{\sf GZ}$} \\ \hline \hline
    ${\cal IC}_{\cal P}^{\sf GZ}(x,y)$ & $=
 \begin{cases}
 \IC_{\cal P}(x)\times \IC_{\cal P}(x) & \mbox{ if }x=y\\
\{\emptyset\}\times \{{\cal A}_{\cal P}\} & \mbox{ otherwise}
\end{cases}
$
\end{tabular}
\caption{Concrete Operators for Disjunctive (Aggregate) Logic Programs}
\label{tab:operators}
\end{table}

\end{document}